\documentclass[11pt]{article}

\usepackage[margin = 0.8in]{geometry}
\usepackage[utf8]{inputenc}
\usepackage{lastpage}
\usepackage{mathtools, amssymb, amsmath, mathrsfs}
\usepackage{enumitem}
\usepackage{natbib}
\usepackage{xcolor}
\usepackage{hyperref}  
\usepackage{booktabs}
\usepackage{adjustbox}
\usepackage{threeparttable}

\usepackage{lastpage}

\hypersetup{hidelinks}

\newtheorem{theorem}{Theorem}[section]
\newtheorem{proposition}[theorem]{Proposition}
\newtheorem{lemma}[theorem]{Lemma}
\newtheorem{corollary}[theorem]{Corollary}

\newtheorem{remark}[theorem]{Remark}

\definecolor{myblue}{HTML}{131D85}
\definecolor{deepred}{RGB}{235,0,20}

\newcommand*\diff{\mathop{}\!\mathrm{d}}

\newcommand{\cF}{\mathcal{F}}

\newcommand{\cN}{\mathcal{N}}

\newcommand{\cR}{\mathcal{R}}

\newcommand{\bE}{\mathbb{E}}

\newcommand{\bN}{\mathbb{N}}

\newcommand{\bP}{\mathbb{P}}

\newcommand{\bR}{\mathbb{R}}

\newcommand{\bZ}{\mathbb{Z}}

\newcommand{\sB}{\mathscr{B}}
\newcommand{\sS}{\mathscr{S}}

\newcommand{\eps}{\varepsilon}
\newcommand{\diam}{\operatorname{diam}}
\newcommand{\relu}{\operatorname{ReLU}}

\newcommand{\Arg}{\operatorname{Arg}}
\newcommand{\Var}{\operatorname{Var}}

\newcommand{\norm}[1]{\left\vert#1\right\vert}
\newcommand{\Norm}[1]{\left\Vert#1\right\Vert}

\newcommand{\BlackBox}{\rule{1.5ex}{1.5ex}}  
\ifdefined\proof
    \renewenvironment{proof}{\par\noindent{\bf Proof\ }}{\hfill\BlackBox\\[2mm]}
\else
    \newenvironment{proof}{\par\noindent{\bf Proof\ }}{\hfill\BlackBox\\[2mm]}
\fi

\title{Sobolev Approximation of Deep ReLU Networks in Log-Barron Space}

\author{Changhoon Song,\thanks{Research Institute of Mathematics, Seoul National University, Seoul, Republic of Korea. Email: \tt{changhoon.song93@snu.ac.kr}}
~Seungchan Ko\thanks{Department of Mathematics, Inha University, Incheon, Republic of Korea. Email: \tt{scko@inha.ac.kr}}
~and ~Youngjoon Hong\thanks{Department of Mathematical Sciences, Seoul National University, Seoul, Republic of Korea. Email: \tt{hongyj@snu.ac.kr}}}

\date{}

\begin{document}

\maketitle

\begin{abstract}
Universal approximation theorems show that neural networks can approximate any continuous function; however, the number of parameters may grow exponentially with the ambient dimension, so these results do not fully explain the practical success of deep models on high-dimensional data. Barron space theory partially addresses this issue by showing that, when the target function has a rapidly decaying Fourier spectrum, a network with $n$ neurons achieves an $O\left(n^{-1/2}\right)$ approximation error. Many existing results, however, are restricted to shallow networks and assume stronger regularity than Sobolev spaces. In this paper, we introduce a log-Barron space that requires a strictly weaker assumption than the classical Barron space and establish approximation bounds for deep narrow networks. For this new function space, we first study the embedding properties and then conduct a statistical analysis via Rademacher complexity. Then we prove that functions in the space can be approximated by deep ReLU networks with explicit depth dependence. We further generalize to the higher-order log-Barron space and obtain an $H^1$ error bound. Our results clarify how deep narrow networks approximate a broader function class than shallow wide networks, through reduced regularity requirements for efficient representation, offering a more precise explanation for the performance of deep architectures and stable use in high-dimensional problems used today.
\end{abstract}


\smallskip

\section{Introduction}\label{sec:introduction}
Deep learning has achieved significant advances in a wide range of scientific and engineering challenges, including computer vision, natural language processing, scientific computing, and physical modeling. These successes have stimulated interest in developing theoretical foundations to explain the strong performance of neural networks. However, the theoretical understanding that underpins these empirical successes remains limited, which necessitates further investigation.

One of the fundamental theoretical underpinnings of neural networks is given by universal approximation theorems; see e.g., \cite{pinkus1999approximation} and references therein. This theorem states that neural networks can approximate any continuous or integrable function within arbitrary error, provided sufficiently large architectures, which is featured in two fundamental perspectives: width, the number of hidden nodes in each layer, and depth, the number of layers. In other words, neural networks can reduce approximation error by increasing width with a fixed number of layers, or by increasing depth with a fixed number of nodes in each layer. We call the former a shallow wide network and the other a deep narrow network.

In a shallow wide framework, \cite{cybenko1989approximation} and  \cite{hornik1991approximation} have proved that two-layer neural networks can approximate a target function by increasing the number of nodes in the hidden layer. However, those theorems guarantee only the existence of such networks and do not address their practical feasibility, such as the choice of model architecture or the number of parameters. Later on, these findings were further extended, providing more constructive and quantitative results. These subsequent studies indicate that, in the worst case, the number of required parameters increases exponentially with input dimension \cite{yarotsky2017error, achour2022general}, implying that neural networks still suffer from the curse of dimensionality. 

Despite the curse of dimensionality, neural networks have proven superior performance in high-dimensional data from various domains. This observation naturally raises the question of which function classes neural networks can approximate efficiently, overcoming the curse of dimensionality. Some prior works suggest that a function with low-frequency structure is more predictable, and therefore, functions with rapidly decaying Fourier spectra are, in general, easier to approximate. \cite{barron1994approximation} formalized this concept by defining a class of functions, called Barron space, in terms of the decay rate of their Fourier transform, thereby identifying functions for which two-layer neural networks can provide dimension-independent and efficient approximation guarantees \cite{Lu2021, Lu2023, li2024two, siegel2024sharp, siegel2022high, wojtowytsch2022representation}. Some notable advances in this direction include \cite{siegel2020approximation}, which showed that if the target function lies in the higher-order Barron space, then a two-layer network achieves approximation in Sobolev norms at a rate that is independent of the input dimension. Provided more layers, \cite{liao2025spectral} extends the result of \cite{barron1994approximation}: by increasing the width, $L$-layer neural networks can achieve dimension-independent approximation error within a function class broader than two-layer cases. Those approximation rates are summarized in Table~\ref{tab:summary}.

\renewcommand{\arraystretch}{2}
\begin{table}[h]
    \centering
    \caption{Comparison of approximation error bounds for neural networks under different assumptions on the target function. $f$ and $F$ denote the target function and neural network, respectively.}
    \begin{adjustbox}{max width=\textwidth}
        \begin{tabular}{c|c|cc|c} 
        \toprule
          & Assumption on function &    Width & Depth &    Approximation error\\
          \midrule
          Barron \cite{barron1994approximation} & $\displaystyle\int_{\bR^d} \norm{\xi}_1\norm{\hat{f}\left(\xi\right)}\diff\xi<\infty$ &    $n$ & $2$ & $\Norm{f-F}_{L^2\left(\Omega\right)} < O\left(n^{-1/2}\right)$\\
          Siegel and Xu \cite{siegel2020approximation} & $\displaystyle\int_{\bR^d} \norm{\xi}_1^{s+1}\norm{\hat{f}\left(\xi\right)}\diff\xi<\infty$ &    $n$ & $2$ & $\Norm{f-F}_{H^s\left(\Omega\right)} < O\left(n^{-1/2}\right)$\\
          Liao and Ming \cite{liao2025spectral} & $\displaystyle\int_{\bR^d} \left(1+\norm{\xi}_1^s\right)\norm{\hat{f}\left(\xi\right)}\diff\xi<\infty$ &    $N$ & $L\le \frac{1}{2s}$ &  $\Norm{f-F}_{L^2\left(\Omega\right)} < O\left(N^{-sL}\right)$\\
          \midrule
          \textbf{Ours (Section~\ref{sec:L2_convergence})} & $\displaystyle\int_{\bR^d} \log_2\left(2+\norm{\xi}_1\right)\norm{\hat{f}\left(\xi\right)}\diff\xi<\infty$ &    $d+4$ & $m$ &     $\Norm{f-F}_{L^2\left(\Omega\right)} < O\left(m^{-1/2}\right)$ \\
          \textbf{Ours (Section~\ref{sec:H1_convergence})} & $\displaystyle\int_{\bR^d} \left(1+\norm{\xi}_1\right)\log_2\left(2+\norm{\xi}_1\right)\norm{\hat{f}\left(\xi\right)}\diff\xi<\infty$ &    $d+4$ & $m$ &     $\Norm{f-F}_{H^1\left(\Omega\right)} < O\left(m^{-1/2}\right)$ \\
          \bottomrule
        \end{tabular}\label{tab:summary}
    \end{adjustbox}
\end{table}

On the other hand, in a deep narrow framework, which reduces approximation error by increasing the number of layers, \cite{gripenberg2003approximation} and \cite{lu2017expressive} established universal approximation theorems with smooth and ReLU activation functions, respectively. Indeed, empirical evidence suggests that increasing depth often yields superior performance relative to increasing width, suggesting that deep narrow networks achieve better approximation than shallow wide networks \cite{poggio2017and}. However, even deep networks cannot fully circumvent the curse of dimensionality, as they still require exponentially many parameters in the worst case. This leads us to invest in the function classes that deep narrow networks can efficiently approximate. A motivating example is illustrated in Figure~\ref{fig:dimension_independent_rate_deep}. 
The left panel illustrates the spectral decay of the target function $f$ by plotting the magnitude of its Fourier coefficients $|\hat f(\xi)|$ against the frequency size $|\xi|_{1}$ on a logarithmic $y$-axis. The blue scatter points correspond to sampled frequencies $\xi$ and the associated values $|\hat f(\xi)|$. The orange dotted markers overlays a reference upper envelope with a logarithmically corrected polynomial decay, of the form
\[
    |\hat f(\xi)| \;\lesssim\; \frac{1}{1+|\xi|_{1}^{d}} \cdot \frac{1}{\bigl(\log_{2}(2+|\xi|_{1})\bigr)^{3}},
\]
which visually confirms that the constructed target exhibits the intended spectral decay pattern (i.e., a log-corrected Barron-type condition but not a classical Barron condition of any order) and motivates the log-Barron assumptions used in our analysis. The right panel plots the approximation error (RMSE, on a logarithmic scale) of the trained deep ReLU network versus depth for several input dimensions $d\in\{11,13,\dots,27\}$. As the depth increases, the RMSE decreases by several orders of magnitude, and the curves corresponding to different dimensions exhibit comparable slopes on the log scale, suggesting that the empirical error decay with respect to depth is essentially insensitive to the ambient dimension in this experiment. This observation motivates establishing a dimension-independent approximation theory with explicit dependence on depth.
Since existing dimension-independent approximation theory has largely been restricted to regimes in which the depth is fixed and the width grows, this observation suggests the need for a new theoretical framework that addresses dimension-independent approximation in the depth-growing setting.
\begin{figure}
    \centering
    \begin{minipage}[h]{0.43\linewidth}
        \includegraphics[width=\linewidth]{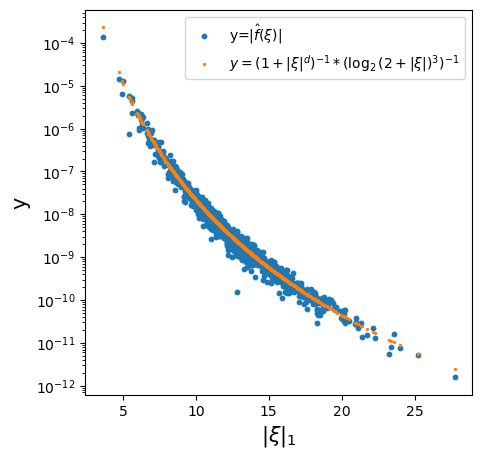}
    \end{minipage}
    \hfill
    \begin{minipage}[h]{0.55\linewidth}
        \includegraphics[width=\linewidth]{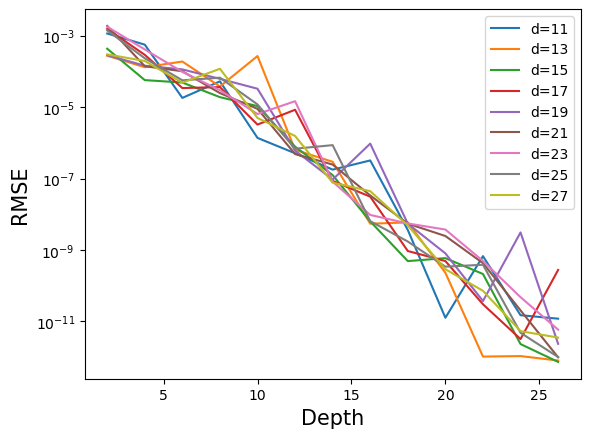}
    \end{minipage}    
    \caption{(Left) Distribution of Fourier spectra of the target function $f$. (Right) 
    Approximation error of a deep neural network for various input dimensions $d$ and depth. 
    In the right panel, the target function is a synthetic function defined as a finite sum of Fourier modes, $f(x) \;=\; \frac{1}{N}\sum_{j=1}^{N} |\hat f(\xi_j)|\, e^{2\pi i\, \xi_j \cdot x}.$ That is, we select a collection of frequencies $\{\xi_j\}_{j=1}^{N}$ (typically by random sampling) and prescribe the amplitudes $|\hat f(\xi_j)|$ to follow the above decay law.
    This illustrates the necessity of establishing a dimension-independent approximation theory with respect to depth.}
    \label{fig:dimension_independent_rate_deep}
\end{figure}

In this paper, we theoretically analyze this with a newly introduced function space, which we refer to as the \textit{log-Barron space}. The space is characterized by slowly decaying Fourier spectra and hence weaker smoothness than classical Barron spaces. To build an initial intuition for this new space, we also provide embedding relations with Sobolev spaces and derive a complexity estimate in terms of the Rademacher complexity \cite{mohri2018foundations}. Within this space, we demonstrate that deep $\relu$ networks with bounded width can achieve a dimension-independent approximation error rate. Since the log-Barron space encompasses a broader class of functions, this result highlights that increasing the depth of the network suffices to attain efficient approximation and extends the class of functions to include those with weaker regularity compared to shallow wide networks. We also extend our analysis to higher-order regularity and achieve a first-order approximation in the Sobolev space. Consequently, we provide a theoretical explanation for the efficient representation of functions with rich high-frequency structures by deep networks, emphasizing that depth, rather than width, is crucial for capturing low-regularity functions. This is also consistent with the experimental findings in Figure \ref{fig:dimension_independent_rate_deep} and supports why deep architectures have shown better performance than wide ones.

To the best of our knowledge, our theorems are the first comprehensive characterization of function spaces that address dimension-independent and efficient approximation in terms of depth. The results indicate that deepening a network can efficiently reduce errors for a function class broader than any other classes corresponding to a shallow wide framework in \cite{barron1994approximation,siegel2020approximation,liao2025spectral}.

Our main contributions are summarized as follows:

\begin{itemize}
    \item We introduce the log-Barron space, a Banach space that encompasses functions with weaker regularity and slower spectral decay than a classical Barron space. We made an initial exploration of the functional-analytic properties of this space, especially its embedding relations to classical Sobolev spaces. We also provide a complexity estimate for the space in terms of Rademacher complexity.
    \item We prove that, despite these substantially weaker regularity assumptions, deep narrow ReLU networks can efficiently approximate functions in this space, achieving, with increasing depth, the same accuracy as wide shallow networks. To the best of our knowledge, this is the first quantitative result demonstrating dimension-independent convergence rates for deep narrow networks as the number of layers increases.
    \item We extend our analysis to the Sobolev approximation, demonstrating that a similar result holds not only for functions themselves but also for their first-order derivatives, thereby enhancing our understanding of more precise approximation.
\end{itemize}

The rest of the paper is organized as follows. Section~\ref{sec:related_works} reviews related works on universal approximation theorems for both wide and deep neural networks, with particular emphasis on Barron's approach. The notations, definitions, and auxiliary results used throughout the paper are introduced in Section~\ref{sec:notations}. We then analyze the embedding properties of the proposed log-Barron space in Section~\ref{sec:log_barron_space}. Section~\ref{sec:L2_convergence} presents the $L^2$ approximation results, and the subsequent section extends these results to the $H^1$ setting. Finally, Section~\ref{sec:conclusion} concludes the paper with a discussion of the main contributions and potential future directions.

\section{Related Works}\label{sec:related_works}
This study investigates the approximation error bounds of deep $\relu$ networks in the log-Barron space, a newly introduced Banach space. In this section, we briefly review related prior works on the universal approximation theorem and approximation error in the Barron space, organized according to shallow, wide, and deep narrow frameworks.

The universal approximation theorem guarantees the existence of a neural network approximating a function in $C\left(\bR^d\right)$ or $L^p\left(\bR^d\right)$ on compact sets. The earliest results on this were introduced in \cite{cybenko1989approximation} and \cite{hornik1991approximation}, which showed that two-layer networks can represent a function $C\left(\bR^d\right)$ or $L^p\left(\bR^d\right)$, provided that the width is sufficiently wide. Building on these theorems, the following studies quantified approximation error in terms of network size. \cite{yarotsky2017error} and \cite{achour2022general} investigated the approximation rates of $\relu$ networks in $C\left(\bR^d\right)$ and $L^p\left(\bR^d\right)$, demonstrating that the number of required parameters increases exponentially with the input dimension $d$. This result suggests that neural networks are subject to the curse of dimensionality, and the mechanisms underlying their strong empirical performance on high-dimensional data remain unclear. 

Consequently, research has focused on identifying function spaces that neural networks can represent efficiently. \cite{barron1994approximation} introduced the Barron space $\sB$, which consists of functions with decaying Fourier transforms characterized by the finite norm $\Norm{f}_\sB=\int_{\bR^d}\left(1+\norm{\xi}_1\right)|\hat{f}\left(\xi\right)| \diff\xi$. It was proved that for any $f\in \sB$ and a compact set $\Omega\subset\bR^d$, a two-layer network $F_n$ of width $n$ achieves $\Norm{f-F_n}_{L^2\left(\Omega\right)}=O\left(n^{-1/2}\right)$, independent of input dimension. \cite{siegel2020approximation} extended this result to higher-order Barron spaces $\sB^s$ equipped with the norm $\Norm{f}_{\sB^s}=\int_{\bR^d}\left(1+\norm{\xi}_1^s\right)|\hat{f}\left(\xi\right)|\diff\xi$, showing that for $f\in \sB^{s+1}$, a two-layer network $F_n$ of width $n$ satisfies $\Norm{f-F_n}_{H^s\left(\Omega\right)}=O\left(n^{-1/2}\right)$ for $s>0$\footnote{In the paper, the Barron norm was defined slightly different as $\Norm{f}_{\sB^s}=\int_{\bR^d}\left(1+\norm{\xi}_1\right)^s|\hat{f}\left(\xi\right)|\diff\xi$. Note that it is straightforward to see that these two definitions are equivalent.}. More recently, \cite{liao2025spectral} studied networks with finite depth $L\ge 2$ and width $N$, demonstrating that the approximation error for $f\in \sB^s$ scales as $O\left(N^{-sL}\right)$ when $0< sL\le\frac{1}{2}$. These results indicate that neural networks reduce approximation error by increasing width while fixing depth, which relaxes the regularity assumption on the target function. Although there is also a probabilistic approach to defining the Barron space via an integral representation (e.g., \cite{ma2022barron} and \cite{chen2024neural}), in this paper, we adopt Barron’s original approach based on the decay of the Fourier transform in order to control frequencies more quantitatively.

On the other hand, \cite{lu2017expressive} and \cite{Kidger_2020} have investigated deep narrow architectures with a bounded number of nodes per layer. Research on deep narrow networks has demonstrated that bounded-width networks can represent a function $C\left(\bR^d\right)$ or $L^p\left(\bR^d\right)$, provided a sufficient number of layers. Constructive and quantitative studies have estimated the required number of parameters in terms of approximation error and the network size, suggesting the curse of dimensionality on deep narrow networks \cite{lu2017expressive, yarotsky2017error}. The analysis of dimension-independent approximation, however, remains lacking. We identify the function space that neural networks can represent efficiently through increasing depth. The target regularity was significantly relaxed by introducing the log-Barron space. As a new function space with weaker regularity requirements than the classical Barron space, the log-Barron space provides theoretical evidence that increasing depth, rather than width, enables more efficient and broad approximation of functions with limited spectral decay.

\section{Notations and Auxiliary Results}\label{sec:notations}
This section introduces the notations and related lemmas which will be used throughout the paper. For input dimension $d \in \bN$, $\sS(\bR^d)$ denotes the space of tempered distributions defined on $\bR^d$. As shown in \cite{liao2025spectral}, considering a target function as a tempered distribution enables the application of the Fourier inversion formula. 

For a vector $\xi\in\bR^d$, we denote the $i$-th component of $\xi$ by $\xi^{(i)}$, hence $\xi=\left(\xi^{(1)}, \xi^{(2)},\ldots,\xi^{(d)}\right)$. The $1$-norm $\norm{\xi}_1$ is the sum of absolute values of all components, $\norm{\xi}_1=\sum_{i=1}^d \norm{\xi^{(i)}}$, and the supremum norm $\norm{\xi}_{\infty}$ is maximum of absolute values, $\norm{\xi}_{\infty}=\max\left\{\norm{\xi^{(i)}}: 1\le i\le d\right\}$. The Euclidean norm of $\xi$ is denoted by $\norm{\xi}=\left(\sum_{i=1}^d\norm{\xi^{(i)}}\right)^{\frac{1}{2}}$. The Barron space $\sB^s$ and its corresponding norm $\Norm{\cdot}_{\sB^s}$ for $s \ge 0$ are defined as 
\begin{align}
    \left\Vert f\right\Vert_{\sB^s} &=  \int_{\bR^d}\left(1+\norm{\xi}_1^s\right)\left\vert\hat{f}\left(\xi\right)\right\vert\diff\xi,\label{eqn:log_barron_1}\\
    \sB^s &= \left\{ f\in \sS(\bR^d): \Norm{f}_{\sB^s}<\infty\right\}.\label{eqn:log_barron_2}
\end{align}
The \textit{log-Barron space} $\sB^{\log}$ and its corresponding norm are defined by
\begin{align}
    \left\Vert f\right\Vert_{\sB^{\log}} &=  \int_{\bR^d}\log_2\left(2+\norm{\xi}_1\right)\left\vert\hat{f}\left(\xi\right)\right\vert\diff\xi,\label{eqn:log_barron_21}\\
    \sB^{\log} &= \left\{ f\in \sS(\bR^d): \Norm{f}_{\sB^{\log}}<\infty\right\}.\label{eqn:log_barron_22}
\end{align}
It is straightforward to verify that $\sB^s \subset \sB^{\log}$ for any $s > 0$. In the next section, it will be shown that $\sB^{\log}$ becomes a Banach space and some embedding properties with respect to the Sobolev spaces $H^s$ will be discussed.

In the proof of the main theorems, a deep neural network is constructed to approximate $f\in \sB^{\log}$ on a compact domain $\Omega\subset\bR^d$. Specifically, a high-frequency cosine function is decomposed into a low-frequency cosine function and a high-frequency piecewise linear function to be determined. We adopt the notation of \cite{liao2025spectral} to denote the repeated function on $\left[0,1\right]$: for a function $g$ on $\left[0,1\right]$, we write
\begin{equation*}
    g_{,n}\left(t\right) \coloneqq g\left(nt\mod 1\right),
\end{equation*}
where for any $z\in\bR$, $\left(z\mod1\right)$ is defined by 
\begin{equation*}
    z\mod 1 \coloneqq z-j\in\left[0,1\right),\quad  j\in\bZ.
\end{equation*}

This notation is used to represent $\cos\left(2\pi nt\right)$ as an integration of $\cos\left(2\pi r\right)\cdot \gamma_{,n}\left(t,r\right)$ over the variable $r$. More precisely, we define $\gamma\left(t,r\right)$ for $t\in\left[0,1\right]$ and $r\in\left[-\frac{1}{2},\frac{1}{2}\right]$ as
\begin{equation}\label{eqn:gamma_tr}
    \gamma\left(t,r\right) = \begin{cases}
        \max\left\{0,r\right\}& t\le \norm{r},\\
        t-\max\left\{0,-r\right\}& \norm{r} < t \le \frac{1}{2},\\
        \gamma\left(1-t,r\right)& t>\frac{1}{2}.
    \end{cases}
\end{equation}

\begin{figure}
    \centering
    \includegraphics[width=0.8\linewidth]{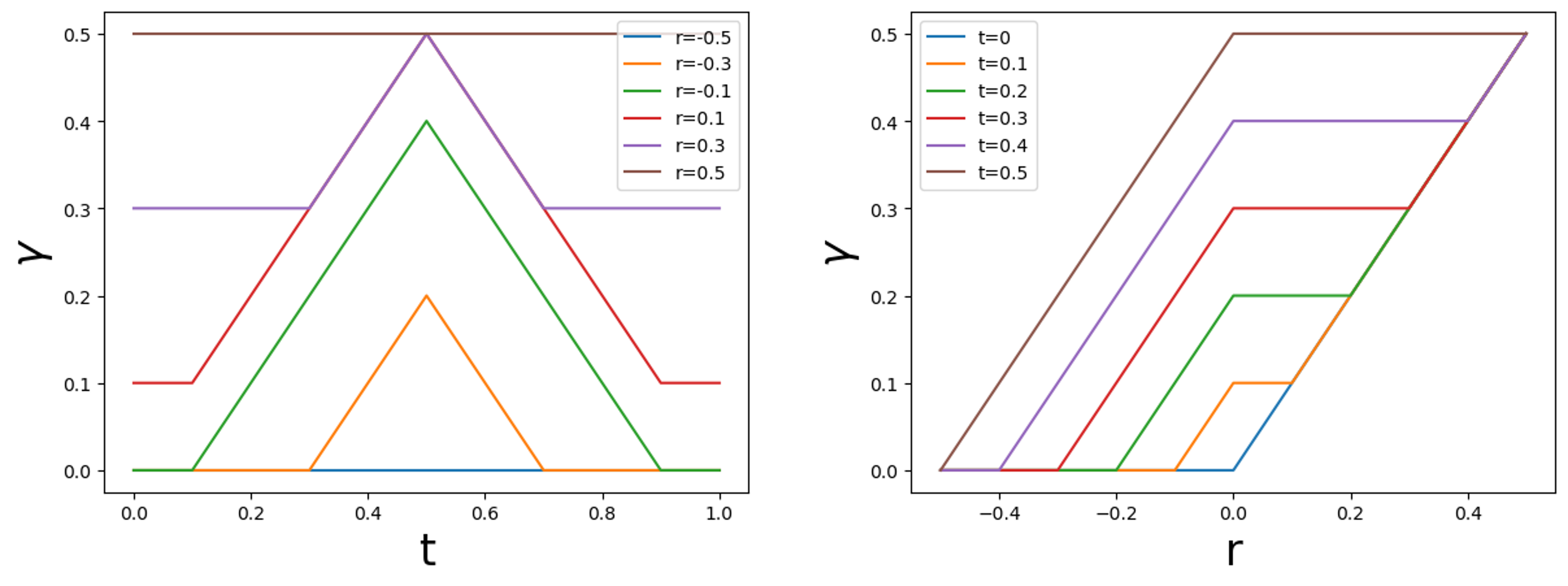}
    \caption{Graph of $\gamma\left(\cdot,r\right)$ and $\gamma\left(t,\cdot\right)$.}
    \label{fig:gamma_graph}
\end{figure}

As illustrated in Figure~\ref{fig:gamma_graph}, $\gamma\left(\cdot, r\right)$ is symmetric to $t=\frac{1}{2}$ for all $r\in\left(-\frac{1}{2},\frac{1}{2}\right)$. Moreover, for given $r\in\left[-\frac{1}{2},\frac{1}{2}\right]$, the function $\gamma\left(\cdot,r\right)$ is weakly differentiable and exactly represented by a $\relu$ network. The following lemma exploits $\gamma$ to represent a high-frequency cosine function.

\begin{lemma}\label{lemma:cos_2pint}
Let $n \in \mathbb{N}$ and $t \in (0,1)$. Define $\gamma_{,n}\left(t,r\right):=\gamma\left(nt\mod 1,r\right)$.
Then there holds
    \begin{equation*}
        -2\pi^2\int_{-\frac{1}{2}}^{\frac{1}{2}}\cos\left(2\pi r\right)\gamma_{,n}\left(t,r\right)\diff r = \cos\left(2\pi nt\right).
    \end{equation*}
\end{lemma}

\begin{proof}
    Since $\cos\left(2\pi nt\right)$ and $\gamma\left(t,r\right)$ are symmetric about $t=\frac{1}{2}$, we may assume $t\in\left(0,\frac{1}{2}\right)$. 
    Moreover, if the equation holds for $n=1$, then the periodicity of cosine deduces the general cases as
    \begin{align*}
        \cos\left(2\pi nt\right) &= \cos\left(2\pi\left(nt\mod 1\right)\right)\\
        &= -2\pi^2\int_{-\frac{1}{2}}^{\frac{1}{2}}\cos\left(2\pi r\right)\gamma\left(nt\mod 1,r\right)\diff r\\
        &= -2\pi^2\int_{-\frac{1}{2}}^{\frac{1}{2}}\cos\left(2\pi r\right)\gamma_{,n}\left(t,r\right)\diff r.
    \end{align*}
    
    To compute the integration over the variable $r$ for given $t$ and $n=1$, we need to consider $\gamma\left(t,r\right)$ as a function on $r$:
    \begin{equation*}
        \gamma\left(t,r\right) = \begin{cases}
            0 & r < -t,\\
            t+r & -t < r \le 0,\\
            t & 0 < r \le t,\\
            r & t < r < \frac{1}{2}.
        \end{cases}
    \end{equation*}
    Then, the direct calculation below concludes the proof:
    \begin{align*}
        &-2\pi^2\int_{-\frac{1}{2}}^{\frac{1}{2}}\cos\left(2\pi r\right)\gamma\left(t,r\right)\diff r \\
        &= -2\pi^2\int_{-t}^0\cos\left(2\pi r\right)\left(t+r\right)\diff r -2\pi^2\int_0^t \cos\left(2\pi r\right)t\diff r -2\pi^2\int_t^{\frac{1}{2}}\cos\left(2\pi r\right)r\diff r\\
        &= \left[ -\pi\left(t+r\right)\sin\left(2\pi r\right)-\frac{1}{2}\cos\left(2\pi r\right)\right]_{-t}^0 -t\pi\left[\sin\left(2\pi r\right)\right]_0^t + \left[ -\pi r\sin\left(2\pi r\right)-\frac{1}{2}\cos\left(2\pi r\right)\right]_t^{\frac{1}{2}}\\
        &= \left[-\frac{1}{2}+\frac{1}{2}\cos\left(2\pi t\right)\right] -t\pi\sin\left(2\pi t\right) + \left[ \frac{1}{2} +t\pi\sin\left(2\pi t\right)+\frac{1}{2}\cos\left(2\pi t\right)\right]\\
        &= \cos\left(2\pi t\right).
    \end{align*}
\end{proof}

In the proof of the main theorems, the integration is approximated by the finite sum over $r$. Since $\gamma\left(t,r\right)$ is symmetric about $t=\frac{1}{2}$, for given $r$, the repeated function $\gamma_{,n}$ can be decomposed using the triangle function
\begin{equation*}
    \beta\left(t\right)\coloneqq\relu\left(2t\right)-2\relu\left(2t-1\right),
\end{equation*}
where $\relu\left(t\right)=\max\left\{0,t\right\}$.
\begin{lemma}[\cite{telgarsky2016benefits,liao2025spectral}]\label{lemma:beta_composition}
    Let $g(t)$ be a function defined on $\left[0,1\right]$ and symmetric about $t=1/2$, then $g_{,n_2}\circ\beta_{,n_1}=g_{,2n_1n_2}$ on $\left[0,1\right]$.
\end{lemma}
\begin{proof} By definition, we first note that
    \begin{align*}
        g_{,n_2}\circ\beta_{,n_1}\left(t\right) &= g\left(n_2\beta_{,n_1}\left(t\right)\mod 1\right)\\
        &= \left\{\begin{array}{ll}
        g\left(2n_1n_2t\mod 1\right), & 0\le \left(n_1t\mod 1\right)\le \frac{1}{2},\\
        g\left(2n_2-2n_1n_2t\mod 1\right), & \frac{1}{2}\le \left(n_1t\mod 1\right)\le 1.
        \end{array}\right.
    \end{align*}
    For the case of $\frac{1}{2}\le \left(n_1t\mod1\right)\le 1$, we have that
    \begin{align*}
        g\left(2n_2-2n_1n_2t\mod1\right) &= g\left(1-\left(2n_2-2n_1n_2t\mod1\right)\right)\\
        &= g\left(1-2n_2+2n_1n_2t\mod1\right)\\
        &= g\left(2n_1n_2t\mod1\right).
    \end{align*}
    Hence, $g_{,n_2}\circ\beta_{,n_1}\left(t\right)=g\left(2n_1n_2t\mod1\right)$ for all $t\in\left[0,1\right]$.
\end{proof}

We conclude this section with a formal definition of neural networks. For given $N$, $L\in\mathbb{N}$, we call a function $F:\bR^d\rightarrow\bR$ a network of width $N$ and depth $L$ if there exist $W_0\in \bR^{N\times d}$, $W_L\in\bR^{1\times N}$, $W_l\in\bR^{N\times N}$ and $\mathbf{b}_l\in\bR^N$ for $1\le l<L$ such that
\begin{align*}
    \mathbf{z}_0 &= x,\\
    \mathbf{z}_l &= \relu \left(W_l\mathbf{z}_{l-1}+\mathbf{b}_l\right),\\
    F(x) &= W_L\mathbf{z}_L.
\end{align*}
The set of all networks of width $N$ and depth $L$ is denoted by $\cN_{L,N}$ and referred to as a class of $(L,N)$-networks.

\section{Log-Barron Space}\label{sec:log_barron_space}
This section investigates properties of the newly introduced log-Barron space. We first prove that the space is a Banach space. Then, we establish an embedding relation between the space and other well-known function spaces. The primary focus is on the inclusion relations with the Sobolev space. Furthermore, we derive an upper bound on the corresponding Rademacher complexity of the proposed space, which leads us to obtain a generalization bound.

It is deduced from elementary inequalities that $\sB^{\log}$ is a vector space equipped with the norm $\Norm{\cdot}_{\sB^{\log}}$.
The inequality $\Norm{f+g}_{\sB^{\log}}\le\Norm{f}_{\sB^{\log}}+\Norm{g}_{\sB^{\log}}$ follows directly from $\norm{\hat{f}\left(\xi\right)+\hat{g}\left(\xi\right)}\le\norm{\hat{f}\left(\xi\right)}+\norm{\hat{g}\left(\xi\right)}$. The Fourier inversion theorem together with the property $\log_2\left(2+\norm{\xi}_1\right)\ge 1$ implies that $\Norm{f}_{\sB^{\log}}=0$ if and only if $f=0$. Therefore, $\sB^{\log}$ is a vector space equipped with the norm $\Norm{\cdot}_{\sB^{\log}}$. The completeness of this space follows from the completeness of the $L^1$ space, which is encapsulated in the following proposition.
\begin{proposition}
    The log-Barron space $\sB^{\log}$ is a Banach space.
\end{proposition}
\begin{proof}
    To prove the completeness, let $\{f_n\}_{n\in\mathbb{N}}$ be a Cauchy sequence in $\sB^{\log}$. Since a sequence $\hat{f_n}$ is Cauchy sequence in $L^1(\bR^d; \diff\mu)$, which is $L^1(\bR^d)$ space equipped with a measure $\diff\mu\left(\xi\right)=\log_2\left(2+\norm{\xi}_1\right)\diff\xi$, there exists $\hat{g}\in L^1(\bR^d;\diff\mu)$ such that 
    \begin{equation*}
        \int_{\bR^d}\log_2\left(2+\norm{\xi}_1\right)\left\vert \hat{f_n}\left(\xi\right)-\hat{g}\left(\xi\right)\right\vert \diff\xi \rightarrow 0 \text{ as }n\rightarrow \infty.
    \end{equation*}
Since $\hat g \in L^1(\mathbb{R}^d)$, we define $g$ via the Fourier inversion formula:
\[
g(x) \coloneqq \int_{\mathbb{R}^d} \hat g(\xi)\, e^{2\pi i \xi \cdot x}\, d\xi .
\]
With this definition, $g \in \sB^{\log}$, and moreover $\|f_n - g\|_{\sB^{\log}} \to 0$ as $n \to \infty$.
\end{proof}

\subsection{Embedding Relations}
The analysis of embeddings with Sobolev spaces begins by considering the embedding relation for the classical Barron space.

\begin{lemma}[\cite{meng2022new, liao2025spectral}]
    If $s_1 > s_2+\frac{d}{2}$, there holds that
    \begin{equation*}
        H^{s_1}(\bR^d)\hookrightarrow \sB^{s_2} \hookrightarrow C^{s_2}(\bR^d).
    \end{equation*}
\end{lemma}
Since $\sB^{s_2}$ embeds into $\sB^{\log}$ for any $s_2>0$, it follows that $H^{s_1}(\bR^d)$ is also embedded in $\sB^{\log}$. This leads to the following corollary.

\begin{corollary}
    If $s>\frac{d}{2}$, then $H^s(\bR^d)\hookrightarrow \sB^{\log}$.
\end{corollary}

Here, the condition $s>\frac{d}{2}$ is essentially optimal. Specifically, for $s\le\frac{d}{2}$, there exists $f\in H^s(\bR^d)$ such that $\Norm{f}_{\sB^{\log}}=\infty$.

\begin{proposition}\label{prop:emb_2}
    If $0\le s\le\frac{d}{2}$, then there exists $f\in H^s(\bR^d)$ with $f\notin \sB^{\log}$.
\end{proposition}
\begin{proof}
    Let $A_k:=\{\xi\in\bR^d: 2^k\le|\xi|<2^{k+1}\}$ denote the dyadic shells.
    Define $\hat{f}$ to be constant on each shell:
    \begin{equation*}
        \hat{f}\coloneqq \sum_{k\ge1}c_k\, \mathbf{1}_{A_k},\quad\text{where } c_k\coloneqq\frac{2^{-(\frac d2+s)k}}{k},
    \end{equation*}
    where $\mathbf{1}_{A_k}$ refers to the indicator function.
    Note that $\hat{f}\notin L^1(\bR^d)$, so we cannot apply the Fourier inversion formula to find such $f$. However, since $\hat{f}\in L^2(\bR^d)$, there exists $f\in L^2(\bR^d)$ whose Fourier transform is $\hat{f}$ by Plancherel's theorem. Then we have
    \begin{equation*}
        \Norm{f}_{H^s}^2
        \lesssim \sum_k \int_{A_k}(1+\norm{\xi}^2)^s \norm{\hat{f}\left(\xi\right)}^2\diff\xi
        \lesssim \sum_k 2^{2sk} c_k^2 \norm{A_k}
        \lesssim \sum_k \frac{1}{k^2}<\infty,\footnote{We denote $f\lesssim g$ or $g\gtrsim f$ if there exists a constant $C$ such that $f\le Cg$.}
    \end{equation*}
    and hence, $f\in H^s(\bR^d)$, where $\norm{A_k}$ is the measure of the dyadic shell $A_k$. On the other hand, for the $\sB^{\log}$ norm, we use the fact $\log_2(2+\norm{\xi}_1) > k$ on $A_k$ and $\norm{A_k}\gtrsim 2^{kd}$ to obtain
    \begin{equation*}
        \Norm{f}_{\sB^{\log}}
    =\sum_{k\ge1}\int_{A_k} \log_2(2+\norm{\xi}_1)\,\norm{\hat{f}\left(\xi\right)}\diff\xi
    \gtrsim \sum_{k\ge1} k\, c_k\,\norm{A_k}\gtrsim \sum_k 2^{-\left(\frac{d}{2}+s\right)k}\cdot 2^{kd}
    =\sum_{k\ge1} 2^{\left(\frac{d}{2} - s\right)k}.
    \end{equation*}
    Since $s\leq\frac d2$, the right-most term diverges and hence, $f\notin \sB^{\log}$.
\end{proof}

Conversely, one may wonder whether the embedding holds in the opposite direction. However, this is generally not the case for any $s\geq0$, which is encapsulated in the following proposition. The idea is to place mass on each shell with amplitude and support size tailored to control the logarithmically weighted $L^1$ while forcing the Sobolev $L^2$ norm to diverge.

\begin{proposition}\label{prop:emb_3}
    For every $s\ge0$ there exists $f\in \sB^{\log}$ with $f\notin H^s(\bR^d)$.
\end{proposition}
\begin{proof}
    Let $A_k:=\left\{\xi\in\bR^d: 2^k\le \norm{\xi}_1 < 2^{k+1}\right\}$ denote the dyadic shells. For a fixed parameter $p>2$, we can choose a measurable set $E_k\subset A_k$ for each $k\in\bN$ with measure
    \begin{equation*}
        \norm{E_k} = k^{-2p}2^{-k}.
    \end{equation*}
    If we define
    \begin{equation*}
        \hat{f}=\sum_{k\ge1} m_k\,\mathbf{1}_{E_k},\quad {\text{where }}m_k=k^p2^k,
    \end{equation*}
    where $\mathbf{1}_{E_k}$ is the indicator function, then we have
    \begin{equation*}
        \int_{\bR^d}\norm{\hat{f}\left(\xi\right)}\diff\xi \leq \sum_{k\geq1}\int_{E_k}\norm{\hat{f}\left(\xi\right)}\diff\xi \leq \sum_{k\geq1}m_k\,\norm{E_k}=\sum_{k\geq1}k^{-p}<\infty.
    \end{equation*}
    Therefore, by the Fourier inversion formula, we can see that such a function $f$ exists in $C_0(\bR^d)\cap L^{\infty}(\bR^d)$. Now, since $\log_2(2+\norm{\xi}_1) \le 2k$ on $A_k$, there holds
    \begin{equation*}
        \Norm{f}_{\sB^{\log}}
        =\int_{\bR^d} \log_2(2+\norm{\xi}_1)\norm{\hat{f}\left(\xi\right)}\diff\xi
        \le \sum_{k\geq1} 2k m_k\,\norm{E_k}\diff\xi
        =2\sum_{k\geq1} k^{1-p}<\infty,
    \end{equation*}
    and hence, $f\in \sB^{\log}$. On the other hand, from $1+\norm{\xi}^2\ge \frac{2}{\sqrt{d}}\norm{\xi}_1$, we have
    \begin{align*}
        \Norm{f}_{H^s}^2
        &= \int_{\bR^d} \left(1+|\xi|^2\right)^s\norm{\hat{f}\left(\xi\right)}^2\diff\xi \\
        &\ge \left(\frac{2}{\sqrt{d}}\right)^s \int_{\bR^d} \norm{\xi}_1^s \norm{\hat{f}\left(\xi\right)}^2\diff\xi \\
        &\ge \left(\frac{2}{\sqrt{d}}\right)^s \sum_{k\ge1} 2^{sk}m_k^2\norm{E_k} \\
        &= \left(\frac{2}{\sqrt{d}}\right)^s \sum_{k\geq1} 2^{\left(s+1\right)k} \\
        &= \infty
    \end{align*}
    for any $s\ge 0$. Therefore, $f\notin H^s(\bR^d)$.
\end{proof}

Collecting all the results above, we can finally obtain the following embedding theorem for the log-Barron space introduced in this paper.

\begin{theorem}\label{emb_thm}
Let $\sB^{\log}(\bR^d)$ be the log-Barron space defined in \eqref{eqn:log_barron_21} and \eqref{eqn:log_barron_22}, and $H^s(\bR^d)$ be the Sobolev space with $s\geq0$. Then we have the following embedding properties.

\begin{enumerate}[label=(\alph*)]
    \item If $s>\frac d2$, then $H^s(\bR^d)\hookrightarrow \sB^{\log}$.
    \item The embedding (a) is essentially sharp: if $0\le s<\frac{d}{2}$, $H^s(\bR^d)\not\subset \sB^{\log}$.
    \item For any $s\ge0$, $\sB^{\log}\not\subset H^s(\bR^d)$.
\end{enumerate}
Therefore, on the range $0\le s \le \frac{d}{2}$ there is no embedding in either direction between $\sB^{\log}$ and $H^s$.
\end{theorem}

\subsection{Estimates of the Rademacher Complexity}
We next quantify the statistical complexity of the log-Barron space $\sB^{\log}$ using Rademacher complexity. In classical statistical learning theory, one typical way to bound the generalization gap, i.e., the discrepancy between the population risk and the empirical risk, is to use the Rademacher complexity. 

To be more specific, given $n$ sample points $x_1,\ldots, x_n\in \Omega$, we consider a ball in the log-Barron space
\begin{equation*}
    \cF_Q \coloneqq \left\{ f\in \sB^{\log}: \Norm{f}_{\sB^{\log}}\le Q\right\}.
\end{equation*}
The \textit{empirical Rademacher complexity} of $\cF_Q$ is then defined as
\begin{equation*}
    \cR_n\left(\cF_Q\right) \coloneqq \bE_{\sigma}\left[\sup_{f\in \cF_Q}\frac{1}{n}\sum_{i=1}^n \sigma_i f\left(x_i\right)\right],
\end{equation*}
where $\sigma_i$ are i.i.d. Rademacher random variables, i.e., $\bP\left(\sigma_i=1\right)=\bP\left(\sigma_i=-1\right)=\frac{1}{2}$. This provides a complementary, sample-based notion of the statistical capacity, measuring how well functions in the class can distinguish random sign patterns on finite samples. The role of the Rademacher complexity is to give a law of large numbers which is uniform over a function class $\mathcal{F}$ and hence bounds the generalization gap \cite{SLT}.

Functions satisfying the classical Barron regularity are known to have low complexity; in particular, their Rademacher complexity can be bounded by $O\left(n^{-1/2}\right)$ \cite{ma2022barron}. It is therefore natural to ask the same question for the new function space we introduce. Since our new space is larger than the classical Barron space, one might anticipate an increased complexity. Nevertheless, we can show that the log-Barron space still admits the same Monte Carlo convergence rate, yielding the same $O\left(n^{-1/2}\right)$ scaling in the corresponding Rademacher complexity bound, and hence comparable generalization guarantees.

We follow the classical Barron-space analysis, with additional care to handle the logarithmic weight in the definition of $\sB^{\log}$. First, we represent functions in $\sB^{\log}$ through the Fourier inversion formula and rewrite the Rademacher complexity as a supremum over a Fourier feature on a dyadic shell. Lastly, we bound the supremum using Dudley's entropy integral, which bounds a centered stochastic process with sub-Gaussian increments on a totally bounded set by a covering number.

Let $\left(T,\rho\right)$ be a metric space and $\eps>0$. We call a set $S\subset T$ is an $\eps$-net of $T$ with respect to $\rho$ if for any $t\in T$, there exists $s\in S$ such that $\rho\left(t,s\right)\le \eps$. The $\eps$-covering number of $T$ is then defined by
\begin{equation*}
    \cN\left(T,\rho,\eps\right)\coloneqq \min\left\{\norm{S} : S\subset T\text{ is an $\eps$-net of $T$ with respect to $\rho$.}\right\}.
\end{equation*}

For a centered stochastic process $\left\{Z\left(t\right)\right\}_{t\in T}$ on a metric space $\left(T,\rho\right)$, we say that $Z$ has sub-Gaussian increments with constant $K>0$ if for all $s,t\in T$ and $u>0$,
\begin{equation}\label{sub_gauss}
    \bP\left(\left\vert Z\left(t\right)-Z\left(s\right)\right\vert\ge u\right) \le 2\exp\left(-\frac{u^2}{K^2\rho\left(t,s\right)^2}\right).
\end{equation}
The following result is to bound the sub-Gaussian process using a covering number, which can be found in many statistical learning books, including Theorem 5.22 in \cite{SLT}.
\begin{proposition}[Dudley's entropy integral]\label{prop:dudley}
    Let $(T,\rho)$ be totally bounded and $Z=\{Z(t)\}_{t\in T}$ is centered, and has sub-Gaussian increments with respect to $\rho$ satisfying $\sup_{t\in T}|Z(t)|\le 1$. Then
\[
\bE\Big[\sup_{t\in T}|Z(t)|\Big]\ \le\ CK\int_{0}^{1}\sqrt{{\log \mathcal{N}(T,\rho,u)}}\,{\rm{d}}u,
\]
where $K$ is the constant in \eqref{sub_gauss}.
\end{proposition}

We use the proposition with a dyadic shell $A_k$ and a corresponding metric $\rho$ to be determined. Our main theorem is as follows.
\begin{theorem}[Bound on the Rademacher complexity]
    Given a set of $n$ data samples $S=\left\{x_1,x_2,\ldots,x_n\right\}$ with $\norm{x_i}_\infty\le 1$, we have
    \begin{equation*}
        \cR_n\left(\cF_Q\right)\le CQ\sqrt{\frac{d}{n}}.
    \end{equation*}
\end{theorem}
\begin{proof}
    By the Fourier inversion theorem, we have
    \begin{equation*}
        f\left(x\right) = \int_{\bR^d}\hat{f}\left(\xi\right) e^{2\pi i\xi\cdot x}\diff\xi.
    \end{equation*}
    Define the weight $\psi_{\xi}\left(x\right)$ and measure $\diff\mu\left(\xi\right)$ by
    \begin{align*}
        \diff\mu\left(\xi\right) &= \log_2\left(2+\norm{\xi}_1\right)\diff\xi,\\
        \psi_{\xi}\left(x\right) &= \left(\log_2\left(2+\norm{\xi}_1\right)\right)^{-1}e^{2\pi i\xi\cdot x},
    \end{align*}
    and we have the representation
    \begin{align*}
        f\left(x\right) &= \int_{\bR^d}\hat{f}\left(\xi\right)\psi_{\xi}\left(x\right)\diff\mu\left(\xi\right),\\
        \int_{\bR^d}\norm{\hat{f}\left(\xi\right)}\diff\mu\left(\xi\right) &= \int_{\bR^d} \log_2\left(2+\norm{\xi}_1\right)\norm{\hat{f}\left(\xi\right)}\diff\xi.
    \end{align*}
    Decompose the frequency space $\bR^d$ into dyadic shells
    \begin{align*}
        A_k &\coloneqq \left\{\xi\in\bR^d : 2^k\le \norm{\xi}_1 < 2^{k+1}\right\},\\
        A_{-1} &= \left\{\xi\in\bR^d : \norm{\xi}_1 < 1\right\},
    \end{align*}
    and we will use the following inequalities on each $A_k$:
    \begin{equation*}
        \frac{1}{\log_2\left(2+\norm{\xi}_1\right)}\le \frac{1}{\log_2\left(2+2^k\right)} \le \frac{2}{k+2}.
    \end{equation*}
    Then, by $L^1(\diff\mu)$ to $L^\infty(\diff\mu)$ duality, we have
    \begin{equation}\label{eqn:RC_main}
        \begin{aligned}
            \cR_n(\cF_Q) &=\bE_{\sigma}\left[\sup_{f\in \cF_Q}\frac{1}{n}\sum_{j=1}^n\sigma_j f\left(x_j\right)\right]\\
            &\le \bE_{\sigma}\left[\sup_{\Norm{\hat{f}}_{L^1(\bR^d)}\le Q}\frac{1}{n}\sum_{j=1}^n\sigma_j \int_{\bR^d}\hat{f}\left(\xi\right)\psi_{\xi}\left(x_j\right)\diff\mu\left(\xi\right)\right]\\
            &\le \bE_{\sigma}\left[\sup_{\Norm{\hat{f}}_{L^1(\bR^d)}\le Q} \int_{\bR^d}\hat{f}\left(\xi\right)\left(\frac{1}{n}\sum_{j=1}^n\sigma_j\psi_{\xi}\left(x_j\right)\right)\diff\mu\left(\xi\right)\right]\\
            &\le Q\bE_{\sigma}\left[\sup_{\xi\in\bR^d}\left(\frac{1}{n}\sum_{j=1}^n\sigma_j \psi_{\xi}\left(x_j\right)\right)\right]\\
            &\le Q\bE_{\sigma}\left[\sup_{k\ge -1}\sup_{\xi\in A_k}\frac{1}{\log_2\left(2+\norm{\xi}_1\right)}\left\vert \frac{1}{n}\sum_{j=1}^n \sigma_j e^{2\pi i\xi\cdot x_j}\right\vert\right]\\
            &\le Q\sup_{k\ge -1}\frac{2}{k+2}\bE_{\sigma}\left[\sup_{\xi\in A_k}\left\vert \frac{1}{n}\sum_{j=1}^n \sigma_j e^{2\pi i\xi\cdot x_j}\right\vert\right].
        \end{aligned}
    \end{equation}
    We use Proposition~\ref{prop:dudley} to bound the real and imaginary parts of the Rademacher process
    \begin{equation*}
        Z\left(\xi\right)\coloneqq \frac{1}{n}\sum_{j=1}^n\sigma_j e^{2\pi i\xi\cdot x_j}.
    \end{equation*}
    First, Hoeffding's inequality (Proposition 2.10 in \cite{SLT}) with random variables 
    \begin{equation*}
        X_j = \frac{1}{n}\sigma_j\left(\cos\left(2\pi \xi\cdot x_j\right)-\cos\left(2\pi\eta\cdot x_j\right)\right),
    \end{equation*}
    with bounds
    \begin{equation*}
        \norm{X_j} \le \frac{2\pi}{n}\norm{\left(\xi-\eta\right)\cdot x_j} \le \frac{2\pi}{n}\norm{\xi-\eta}_1
    \end{equation*}
    implies that the real part has sub-Gaussian increments with respect to $\rho\left(\xi,\eta\right)=\norm{\xi-\eta}_1$:
    \begin{align*}
        \bP\left(\left\vert\textrm{Re}\left(Z\left(\xi\right)-Z\left(\eta\right)\right)\right\vert\ge t\right) &= \bP\left(\left\vert \textrm{Re}\left(\sum_{j=1}^n\frac{1}{n}\sigma_j\left(e^{2\pi i \xi\cdot x_j}-e^{2\pi i\eta\cdot x_j}\right)\right)\right\vert\ge t\right)\\
        &= \bP\left(\left\vert \sum_{j=1}^n X_j - \bE_\sigma\left[X_j\right]\right\vert\ge t\right)\\
        &\le 2\exp\left(-\frac{t^2}{\left(\sum_{j=1}^n \frac{2\pi}{n}\norm{\xi-\eta}_1\right)^2}\right)\\
        &= 2\exp\left(-\frac{nt^2}{4\pi^2\rho\left(\xi,\eta\right)^2}\right).
    \end{align*}
    On the other hand, since $S=\left\{ i\in\bZ : -2^{k+1} < i < 2^{k+1}\right\}$ is a 1-net of $\left(-2^{k+1},2^{k+1}\right)\subset \bR$, $S^d$ is a $d$-net of $A_k$ with respect to $\rho$. Hence, we have
    \begin{equation*}
        \cN\left(A_k,\rho,u\right) \le \left(\frac{2\pi}{u}\right)^d\norm{S^d} \le \left(\frac{2\pi2^{k+2}}{u}\right)^d,
    \end{equation*}
    and Proposition~\ref{prop:dudley} yields
    \begin{equation}\label{eqn:RC_real}
        \bE_\sigma\left[\sup_{\xi\in A_k}\left\vert\textrm{Re} Z\left(\xi\right)\right\vert\right] \le C\int_0^1\sqrt{\frac{\log\cN\left(A_k,\rho,u\right)}{n}}\diff u\le C\sqrt{\frac{d}{n}}\int_0^1\sqrt{\log\left(\frac{2\pi 2^{k+2}}{u}\right)}\diff u.
    \end{equation}
    With the change of variable $a_k=\log\left(2\pi 2^{k+2}\right)$ and $v=a_k-\log u$, we have
    \begin{equation}\label{eqn:shell-contrib}
        \begin{aligned}
        \int_0^1\sqrt{\log\left(\frac{2\pi2^{k+2}}{u}\right)}\diff u &= \int_0^1\sqrt{v}\diff u \\
        &= \int_{a_k}^{\infty} \sqrt{v} e^{a_k-v}\diff v \\
        &= \int_0^\infty \sqrt{t+a_k}e^{-t}\diff t \qquad (t=v-a_k) \\
        &\le \int_0^\infty \sqrt{a_k}\sqrt{t+1}e^{-t}\diff t \\
        &\le C\sqrt{k+2}.
        \end{aligned}
    \end{equation}
    Plugging \eqref{eqn:shell-contrib} and \eqref{eqn:RC_real} gives
    \begin{equation*}
        \sup_{k\ge -1}\frac{2}{k+2}\bE_{\sigma}\left[\sup_{\xi\in A_k}\left\vert \textrm{Re}Z\left(\xi\right)\right\vert\right] \le \sup_{k\ge-1}\frac{2}{\sqrt{k+2}}C\sqrt{\frac{d}{n}} \le C\sqrt{\frac{d}{n}}.
    \end{equation*}
    Similarly, Hoeffding's inequality with random variables
    \begin{equation*}
        Y_j = \frac{1}{n}\sigma_j\left(\sin\left(2\pi \xi\cdot x_j\right)-\sin\left(2\pi\eta\cdot x_j\right)\right),
    \end{equation*}
    and the following steps bound the imaginary part:
    \begin{equation*}
        \sup_{k\ge -1}\frac{2}{k+2}\bE_{\sigma}\left[\sup_{\xi\in A_k}\left\vert \textrm{Im}Z\left(\xi\right)\right\vert\right] \le C\sqrt{\frac{d}{n}}.
    \end{equation*}
    Combining all the inequalities, we attain
    \begin{align*}
        \cR_n\left(\cF_Q\right) &\le Q\sup_{k\ge-1}\frac{2}{k+2}\bE_{\sigma}\left[\sup_{\xi\in A_k}\left\vert Z\left(\xi\right)\right\vert \right]\\
        &\le Q\sup_{k\ge-1}\frac{2}{k+2}\left(\bE_{\sigma}\left[\sup_{\xi\in A_k}\left\vert \textrm{Re}Z\left(\xi\right)\right\vert \right] + \bE_{\sigma}\left[\sup_{\xi\in A_k}\left\vert \textrm{Im}Z\left(\xi\right)\right\vert \right]\right)\\
        &\le CQ\sqrt{\frac{d}{n}}.
    \end{align*}
\end{proof}
As noted in \cite{bartlett2002rademacher} and \cite{ma2022barron}, this implies that functions in the log-Barron spaces can be efficiently learned.

\section{$L^2$ convergence}\label{sec:L2_convergence}
As in the previous sections, a simple observation gives that if $f\in \sB^{\log}$, it follows that $\hat{f}\in L^1(\bR^d)$ and hence, $f\in C(\bR^d)\cap L^{\infty}(\bR^d)$. In particular, $f\in \sB^{\log}$ implies $L^2(\Omega)$ for any compact set $\Omega$. In this section, we establish a dimension-independent $L^2$ error bound of a deep $\relu$ network on $\Omega$ for a target function $f\in \sB^{\log}$, as the number of layers grows.

Consistent with prior research on Barron spaces, the idea is to represent the target function as the expectation of a random variable, and the expected value is approximated using the sample mean of independent and identically distributed random samples. As the number of Monte Carlo samples increases, the empirical mean converges to the true expectation with high probability, demonstrating the existence of realizations that achieve a sufficiently small approximation error. Similarly, if neural networks are constructed to realize these samples, the ensemble, defined as the empirical mean of such sub-networks, provides an accurate approximation of the target function. The following lemma is then employed to construct a single deep narrow network that aggregates these realizations, thereby establishing the desired approximation result.
\begin{lemma}\label{lemma:wide_to_deep}
    Assume that $\Omega\subset\bR^d$ is compact, and let $m\in\bN$, $L_i,N_i\in\bN$ and $F_i\in\cN_{L_i,N_i}$ be given for $i=1,\ldots,m$. Then there exists a $\left(L,d+N+1\right)$-network $F$ such that
    \begin{equation*}
        F\left(x\right) = \frac{1}{m}\sum_{i=1}^m F_i\left(x\right), \quad \text{ for all }x\in\Omega,
    \end{equation*}
    where $L=\sum_i L_i$ and $N=\max_i N_i$.
\end{lemma}
\begin{proof}
    With an affine transform if needed, we may assume $\Omega \subset [0,1]^d$. Moreover, without loss of generality, let $N = N_1$. Suppose that, for each $i$, $W_{i,l}$ and $\mathbf{b}_{i,l}$ denote the weight matrices and bias vectors, respectively, for the network $F_i$:
    \begin{equation*}
        F_i\left(x\right) = W_{i,L_i}\cdot\relu\left(W_{i,L_i-1}\cdot\relu\left(\cdots\right)+\mathbf{b}_{i,L_i-1}\right).
    \end{equation*}
    For $0< j< L_1$, set the weights $W_0\in\bR^{d\times(d+N+1)}$, $W_j\in\bR^{(d+N+1)\times(d+N+1)}$ and biases $\mathbf{b}_j\in\bR^{d+N+1}$ so that
    \begin{equation*}
        W_0 = \begin{pmatrix}
            \textrm{Id}_d \\ W_{1,0} \\ \mathbf{0}
 ,       \end{pmatrix}, \quad
        W_j = \begin{pmatrix}
            \textrm{Id}_d & \mathbf{0} & \mathbf{0}\\ \mathbf{0} & W_{1,j} & \mathbf{0}\\ \mathbf{0} & \mathbf{0} & 1
        \end{pmatrix}, \quad \mathbf{b}_j = \begin{pmatrix}
            \mathbf{0}_d \\ \textbf{b}_{1,j-1} \\ 0
        \end{pmatrix}.
    \end{equation*}
    For each $i\ge2$ and $0< j< L_i$, we set the weights as
    \begin{equation*}
        W_{\sum_{k=1}^{i-1}L_k} = \begin{pmatrix}
            \textrm{Id}_d & \mathbf{0} & \mathbf{0} \\ W_{i,0} & \mathbf{0} & \mathbf{0} \\ \mathbf{0} & \frac{1}{m}W_{i-1,L_{i-1}} & 1
        \end{pmatrix}, 
        \quad W_{\sum_{k=1}^{i-1}L_k+j} = \begin{pmatrix}
            \textrm{Id}_d & \mathbf{0} & \mathbf{0}\\ \mathbf{0} & W_{i,j} & \mathbf{0}\\ \mathbf{0} & \mathbf{0} & 1
        \end{pmatrix},
    \end{equation*}
    and the biases as
    \begin{equation*}
        \mathbf{b}_{\sum_{k=1}^{i-1}L_i+j} = \begin{pmatrix}
            \mathbf{0} \\ \textbf{b}_{i,j-1} \\ 0
        \end{pmatrix}.
    \end{equation*}
    Then, it is straightforward to verify that the network $F$ given by
    \begin{equation*}
        \mathbf{z}_0 = x,\quad \mathbf{z}_l = \relu\left(W_l\mathbf{z}_{l-1}+\mathbf{b}_l\right), \quad \mathbf{z}_L = W_L\mathbf{z}_{L-1}
    \end{equation*}
    for $\ell=1,\ldots, L$, is the desired $(L,d+N+1)$-network.
\end{proof}

This lemma allows us to consider an ensemble model consisting of sub-networks of different sizes. While high-frequency components would normally necessitate deeper sub-networks for accurate approximation, the amplitude decay of $f$ for $f\in \sB^{\log}$ allows us to use fewer deep sub-networks and more shallow ones. As a result, the total depth of all sub-networks can be bounded. We first show the result for a compact set in the unit cube $\left[0,1\right]^d$ and generalize to a general case.
\begin{theorem}\label{thm:L2_convergence_unit}
    Suppose that $f\in \sB^{\log}$. For any $m\in\bN$ and a compact set $\Omega\subset\left[0,1\right]^d$, there exist $\relu$ networks $F_1,\ldots,F_m$ of width 3 and depths $L_1,\cdots,L_m$ such that
    \begin{equation*}
        \left\Vert f - \frac{1}{m}\sum_{i=1}^m F_i \right\Vert_{L^2\left(\Omega\right)}^2\le \frac{3\pi^4}{m}\norm{\Omega}\Norm{f}_{\sB^0}^2
    \end{equation*}
    and
    \begin{equation*}
        \sum_{i=1}^m L_i \le 5 m\frac{\Norm{f}_{\sB^{\log}}}{\Norm{f}_{\sB^0}}.
    \end{equation*}
\end{theorem}
\begin{proof}
    By the Fourier inversion theorem, we have
    \begin{align*}
        f\left(x\right)&=\int_{\bR^d}\hat{f}\left(\xi\right)e^{2\pi i\xi\cdot x}\,\diff \xi \\
        &=\int_{\bR^d}\left\vert \hat{f}\left(\xi\right)\right\vert e^{i\Arg \hat{f}\left(\xi\right)+2\pi\xi\cdot x} \diff\xi\\
        &= \int_{\bR^d}\left\vert\hat{f}\left(\xi\right)\right\vert\cos\left(2\pi\left(\xi\cdot x+\theta\left(\xi\right)\right)\right)\,\diff \xi,
    \end{align*}
    where $\theta\left(\xi\right)=\frac{1}{2\pi}\Arg\hat{f}\left(\xi\right) + k\left(\xi\right)$ for some $k\left(\xi\right)\in\bZ$. Specifically, given $\xi=\left(\xi^{(1)},\ldots,\xi^{(d)}\right)$ and $x\in\left[0,1\right]^d$, we consider $k\left(\xi\right)=\lceil\sum_{\xi^{(j)}<0}\left\vert\xi^{(j)}\right\vert-\frac{1}{2\pi}\Arg\hat{f}\left(\xi\right)\rceil$ so that $\xi\cdot x+\theta\left(\xi\right)\in\left[0,\norm{\xi}_1+1\right]$.

    Given $\xi$, we let $n_{\xi}=2^{\lceil\log_2\left(2+\norm{\xi}_1\right)\rceil}$ and $t_{\xi}\left(x\right)=\frac{1}{n_{\xi}}\left(\xi\cdot x+\theta\left(\xi\right)\right)$ to make $t_{\xi}\left(x\right)\in\left[0,1\right]$ and
    \begin{align*}
        f\left(x\right)&=\int_{\bR^d}\left\vert\hat{f}\left(\xi\right)\right\vert\cos\left(2\pi n_\xi t_\xi\left(x\right)\right)\,\diff \xi\\
        &= -2\pi^2 \int_{\bR^d}\left\vert\hat{f}\left(\xi\right)\right\vert \int_{-\frac{1}{2}}^{\frac{1}{2}} \cos\left(2\pi r\right)\gamma_{,n_\xi}\left(t_\xi\left(x\right),r\right) \diff r \diff\xi,
    \end{align*}
    where the second equality comes from Lemma~\ref{lemma:cos_2pint}.

    Next, we define a probability measure $\mu\left(\xi,r\right)$ on $\bR^d\times\bR$ and $F\left(\cdot;\xi,r\right):\bR^d\rightarrow\bR$ as follows:
    \begin{align*}
        \diff\mu\left(\xi,r\right) &= \frac{1}{\Norm{f}_{\sB^0}}\mathbf{1}_{\left[-\frac{1}{2},\frac{1}{2}\right]}\left(r\right) \norm{\hat{f}\left(\xi\right)} \diff\xi \diff r,\\
        F\left(x;\xi,r\right) &= -2\pi^2\Norm{f}_{\sB^0}\cos\left(2\pi r\right)\gamma_{,n_{\xi}}\left(t_{\xi}\left(x\right),r\right).
    \end{align*}
    Then $f$ is the expectation value of $F$ with respect to the probability measure $\diff\mu$:
    \begin{equation*}
        f\left(x\right)=\bE_{\mu}\left[F\left(x;\xi,r\right)\right].
    \end{equation*}
    Let $\left\{\left(\xi_i,r_i\right)\right\}_{i=1}^m$ be i.i.d. random samples and $\bar{F}$ be the sample mean of $F\left(x;\xi_i,r_i\right)$:
    \begin{equation*}
        \bar{F}\left(x\right)=\frac{1}{m}\sum_{i=1}^m F\left(x;\xi_i,r_i\right).
    \end{equation*}
    Now consider a random variable $X=\left\Vert f-\bar{F}\right\Vert_{L^2\left(\Omega\right)}^2$. Denoting the distribution of i.i.d. samples $\left\{\left(\xi_i,r_i\right)\right\}_{i=1}^m$ by $\mu^m$, we have
    \begin{align*}
        \bE_{\mu^m}\left[X\right] &= \int\int_{\Omega}\bigg\vert f\left(x\right)-\frac{1}{m}\sum_{i=1}^m F\left(x;\xi_i,r_i\right)\bigg\vert^2 \diff x \diff\mu^m\left(\left\{\xi_i,r_i\right\}\right)\\
        &= \int_{\Omega}\int \bigg\vert \bE_{\mu^m}\left[\frac{1}{m}F\left(x;\xi_i,r_i\right)\right]-\frac{1}{m}\sum_{i=1}^m F\left(x;\xi_i,r_i\right)\bigg\vert^2 \diff\mu^m\left(\left\{\xi_i,r_i\right\}\right) \diff x\\
        &= \int_{\Omega} \Var_{\mu^m}\bigg[\frac{1}{m}\sum_{i=1}^m F\left(x;\xi,r\right)\bigg] \diff x\\
        &= \frac{1}{m} \int_{\Omega} \Var_{\mu}\left[F\left(x;\xi,r\right)\right]  \diff x\\
        &\le \frac{1}{m}\int_{\Omega} \bE_{\mu}\left[F\left(x;\xi,r\right)^2\right] \diff x\\
        &\le \frac{1}{m}\int_{\Omega} \sup_{\xi,r} \left\vert F\left(x;\xi,r\right)\right\vert^2 \diff x\\
        &= \frac{\pi^4}{m} \norm{\Omega}\Norm{f}_{\sB^0}^2.
    \end{align*}
    Note here that for any $a>0$, Markov's inequality induces
    \begin{equation*}
        \bP_{\mu^m}\left[X\ge a\right] \le \frac{1}{a}\bE_{\mu^m}\left[X\right].
    \end{equation*}
    In particular, the choice $a=\left(2+\eps_1\right)\bE_{\mu^m}\left[X\right]$ implies that 
    \begin{equation}\label{eqn:Markov_L2error}
        \bP_{\mu^m}\left[ X < \frac{\left(2+\eps_1\right)\pi^4}{m}\norm{\Omega}\Norm{f}_{\sB^0}^2\right] \ge \bP_{\mu^m}\left[X < \left(2+\eps_1\right)\bE_{\mu^m}\left[X\right]\right] = \frac{1+\eps_1}{2+\eps_1}.
    \end{equation}

    On the other hand, once the sample $\left\{\left(\xi_i,r_i\right)\right\}^m_{i=1}$ is chosen, we see that
    \begin{equation*}
        F\left(x;\xi_i,r_i\right) = -2\pi^2\Norm{f}_{\sB^0}\cos\left(2\pi r_i\right) \gamma_{,n_{\xi_i}}\left(t_{\xi_i}\left(x\right),r_i\right)
    \end{equation*}
    is the composition of $\gamma_{,n_{\xi_i}}\left(\cdot,r_i\right)$ and the affine function 
    $t_{\xi_i}\left(x\right)$ up to constant. Thus, it suffices to construct a network that represents $\gamma_{,n_{\xi_i}}\left(\cdot, r_i\right)$ for each $i$.
    By Lemma~\ref{lemma:beta_composition}, we know
    \begin{equation*}
        \gamma_{,n_{\xi_i}}\left(\cdot,r_i\right) = \gamma_{, 2^{\lceil\log_2\left(2+\norm{\xi}_1\right)\rceil}}\left(\cdot,r_i\right) = \gamma\left(\cdot,r_i\right)\circ \overbrace{\beta\circ\beta\cdots\circ\beta}^{L_{\xi_i}},
    \end{equation*}
    where $L_{\xi}=\lceil\log_2\left(2+\norm{\xi}_1\right)\rceil$. The network for this construction is illustrated in Figure~\ref{fig:construction_single}. 
    \begin{figure}
        \centering
        \includegraphics[width=0.8\linewidth]{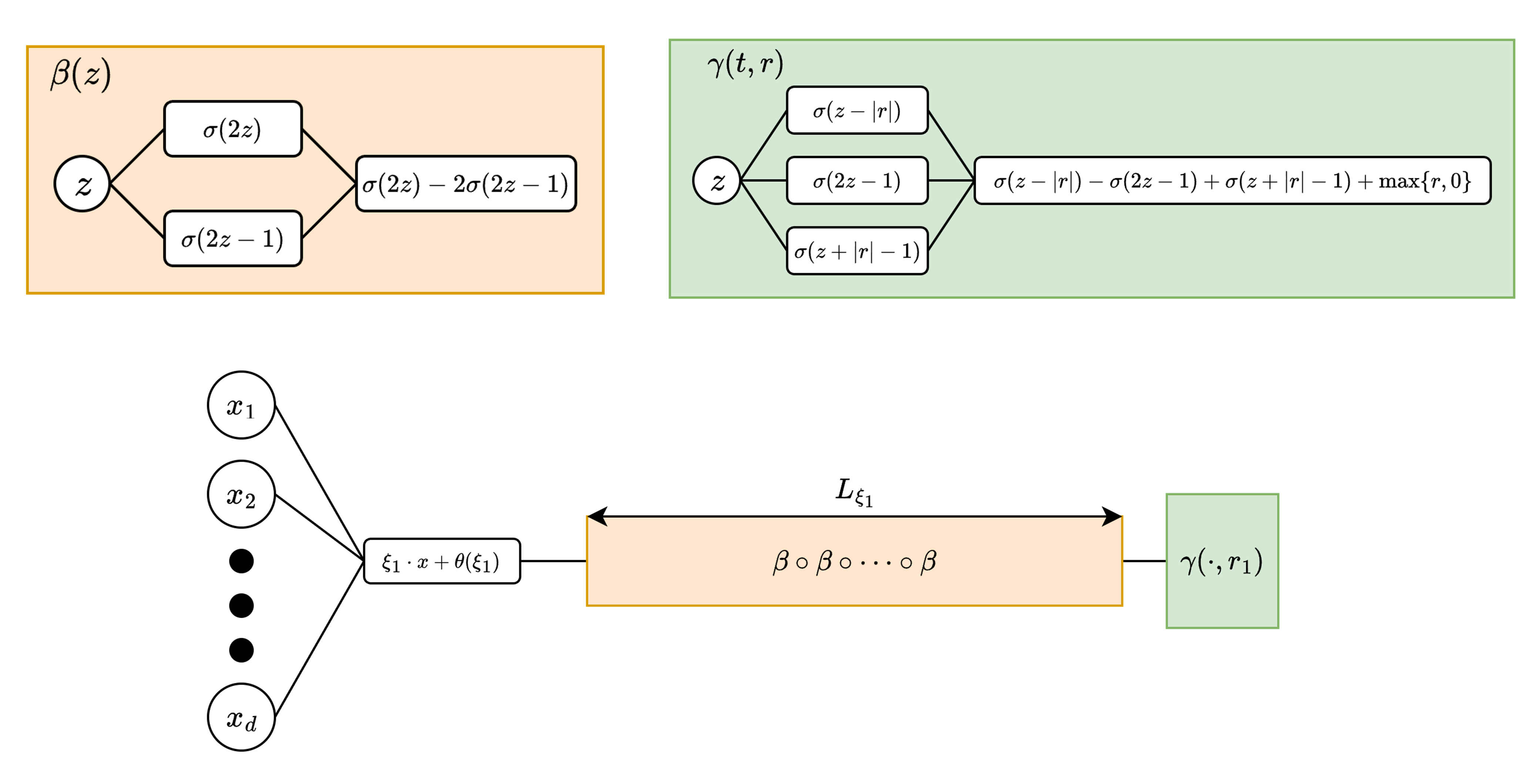}
        \caption{Construction of $\beta$, $\gamma\left(\cdot,r\right)$ and $F\left(x;\xi_i,r_i\right)$. $\sigma$ denotes the $\relu$ activation function.}
        \label{fig:construction_single}
    \end{figure}
    Since $\beta$ and $\gamma\left(\cdot,r_i\right)$ are represented by one linear layer with 2 and 3 nodes respectively, a network of width $3$ and depth $L_{\xi_i}+1$ constructs $\gamma_{,n_{\xi_i}}\left(\cdot,r_i\right)$. Then we can estimate the expected sum of depths $L=\sum_{i=1}^m L_{\xi_i}$ :
    \begin{align*}
        \bE_{\mu^m}\left[L\right] &= \bE_{\mu^m}\bigg[\sum_{i=1}^m L_{\xi_i}\bigg]\\
        &= m \bE_{\mu^m}\bigg[\frac{1}{m}\sum_{i=1}^m L_{\xi_i} \bigg]\\
        &= m \bE_{\mu}\left[L_{\xi}\right] \\
        &\le 2m \bE_{\mu}\left[ \log_2\left(2+\norm{\xi}_1\right)\right]\\
        &= \frac{2m}{\Norm{f}_{\sB^0}}\int_{\bR^d}\log_2\left(2+\norm{\xi}_1\right)\norm{\hat{f}\left(\xi\right)} \diff \xi\\
        &= 2m\frac{\Norm{f}_{\sB^{\log}}}{\Norm{f}_{\sB^0}}.
    \end{align*}
    Similarly to the inequality~\eqref{eqn:Markov_L2error}, Markov's inequality deduces
    \begin{equation}\label{eqn:Markov_L2depth}
        \bP_{\mu^m}\left[ L < \frac{2\left(2+\eps_2\right)m}{\Norm{f}_{\sB^0}}\Norm{f}_{\sB^{\log}}\right] \ge \bP_{\mu^m}\left[L < \left(2+\eps_2\right)\bE_{\mu^m}\left[L\right]\right] \ge 1-\frac{1}{2+\eps_2} = \frac{1+\eps_2}{2+\eps_2}.
    \end{equation}
    Both \eqref{eqn:Markov_L2error} and \eqref{eqn:Markov_L2depth} should hold with positive probability. To be more specific, we have
    \begin{align*}
        & \bP_{\mu^m}\left[\left(X < \frac{\left(2+\eps_1\right)\pi^4}{m}\norm{\Omega}\Norm{f}_{\sB^0}^2\right) \text{ and } \left(L < \frac{2\left(2+\eps_2\right)m}{\Norm{f}_{\sB^0}}\Norm{f}_{\sB^{\log}}\right)\right]\\
        &= \bP_{\mu^m}\left[X < \frac{\left(2+\eps_1\right)\pi^4}{m}\norm{\Omega}\Norm{f}_{\sB^0}^2\right] + \bP_{\mu^m}\left[L < \frac{2\left(2+\eps_2\right)m}{\Norm{f}_{\sB^0}}\Norm{f}_{\sB^{\log}}\right]\\
        &\hspace{4mm} - \bP_{\mu^m}\left[\left(X < \frac{\left(2+\eps_1\right)\pi^4}{m}\norm{\Omega}\Norm{f}_{\sB^0}^2\right) \text{ or } \left(L < \frac{2\left(2+\eps_2\right)m}{\Norm{f}_{\sB^0}}\Norm{f}_{\sB^{\log}}\right)\right]\\
        &\ge \frac{1+\eps_1}{2+\eps_1}+\frac{1+\eps_2}{2+\eps_2}-1\\
        &> 0.
    \end{align*}
    As $\eps_1,\eps_2$ can be arbitrary positive numbers, the choice $\eps_1=1$ and $\eps_2=0.5$ concludes the proof.
\end{proof}

\begin{remark}
    In the above proof, we adopt a statistical approach motivated by \cite{liao2025spectral}: represent $f$ as the expected value of a random variable, approximate the expectation by sample mean, and construct a network computing each sample. We define $\gamma$ in equation~\eqref{eqn:gamma_tr} so that each $\gamma\left(\cdot,r\right)$ is weakly differentiable and exactly represented by $\relu$ network of width $3$. This derives two major advantages: the Sobolev approximation and enabling merging into one network with a fixed width. The Sobolev approximation is discussed in the next section.
\end{remark}

Combined with the previous lemma, we attain an error bound for a deep narrow network on $\Omega\subset\left[0,1\right]^d$, whose convergence rate is independent of the input dimension.
\begin{corollary}\label{cor:L2_convergence_unit}
    Suppose $f\in \sB^{\log}$. For a compact set $\Omega\subset\left[0,1\right]^d$ and any $m\in\bN$, there exists a ReLU network $F$ of width $d+4$ and depth $6m \frac{\Norm{f}_{\sB^{\log}}}{\Norm{f}_{\sB^0}}$ such that
    \begin{equation*}
        \left\Vert f - F\right\Vert_{L^2\left(\Omega\right)} \le \frac{2\pi^2}{\sqrt{m}}\norm{\Omega}^{\frac{1}{2}}\Norm{f}_{\sB^0}.
    \end{equation*}
\end{corollary}
\begin{proof}
    We can obtain the desired result by merging $F_1,F_2,\ldots,F_m$ given in the Theorem~\ref{thm:L2_convergence_unit}, based on Lemma~\ref{lemma:wide_to_deep}. The construction is represented in Figure~\ref{fig:construction_deep}.
\end{proof}

\begin{figure}[h]
    \centering
    \includegraphics[width=0.8\linewidth]{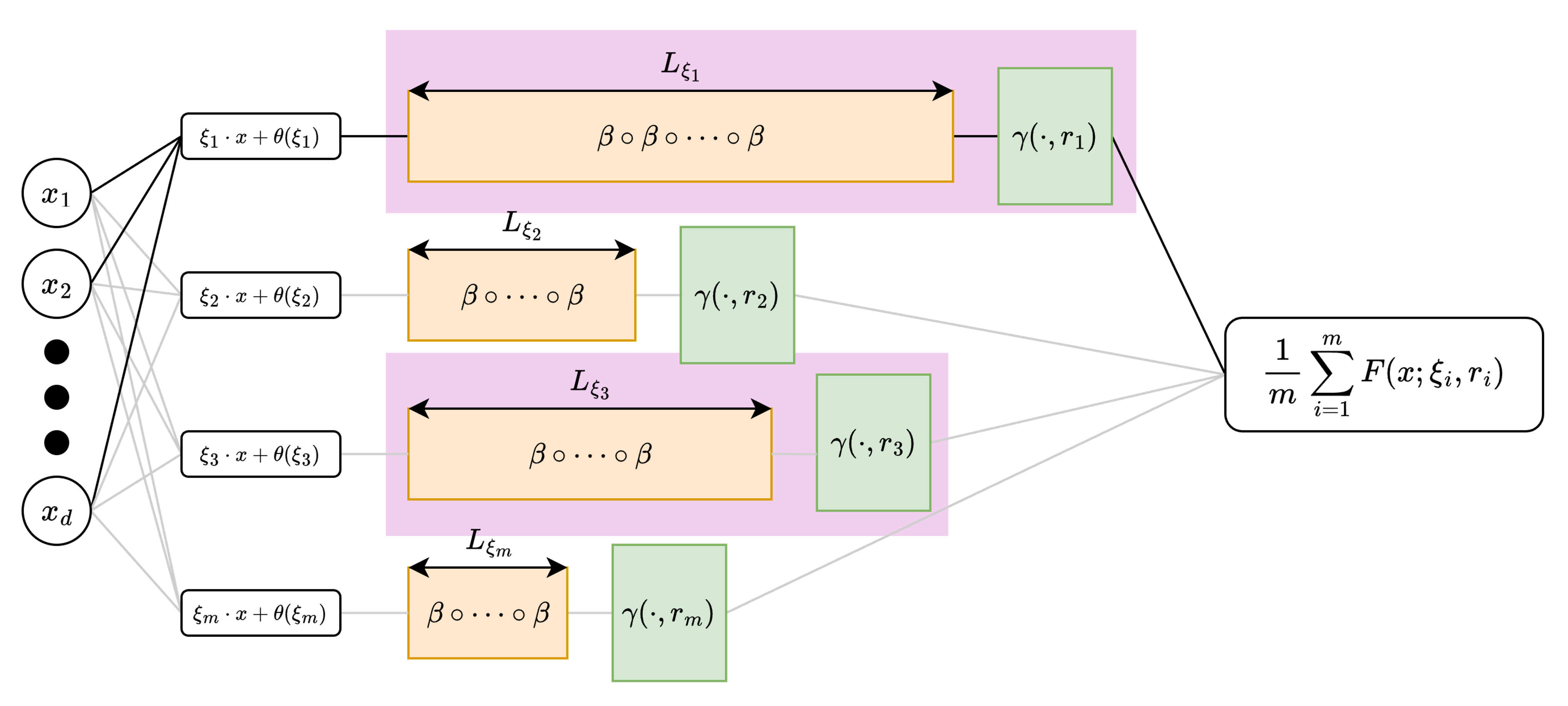}
    \includegraphics[width=0.8\linewidth]{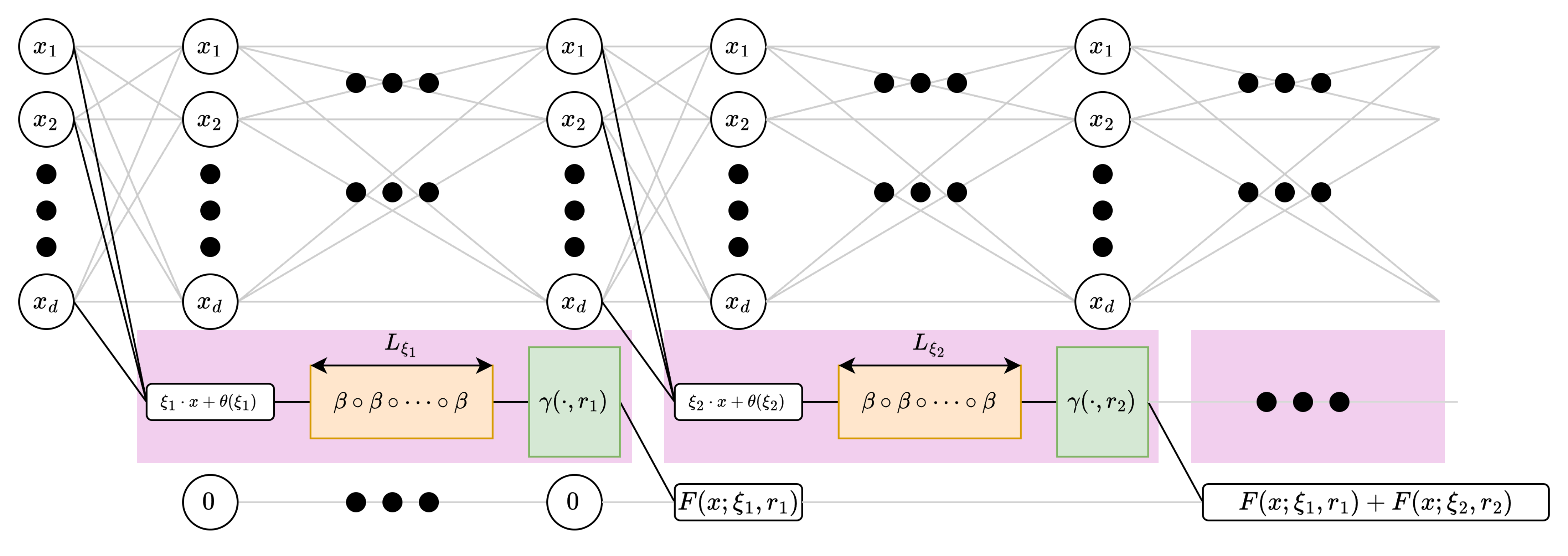}
    \caption{Network construction for Theorem~\ref{thm:L2_convergence_unit}(top) and Corollary~\ref{cor:L2_convergence_unit}(bottom). }
    \label{fig:construction_deep}
\end{figure}
    
\begin{remark}
    Although \cite{liao2025spectral} showed that an $(L,N)$-network achieves an approximation error rate of $N^{-sL}$ for $f\in \sB^s$, this result is valid only for $0<sL\le\frac{1}{2}$ with a fixed $L$, and the constant term depends on $\Norm{f}_{\sB^s}$. Therefore, the depth $L$ is constrained for a given $f\in \sB^s$, and the error remains sensitive to the decay rate of the Fourier amplitude of $f$. Additionally, $(L,N)$ networks require $O\left(N^2L\right)$ parameters to achieve an $N^{-\frac{1}{2}}$ error. In contrast, the network described in Corollary~\ref{cor:L2_convergence_unit} has $O\left(m\right)$ parameters, and the error is determined by the lower-regularity term $\Norm{f}_{\sB^0}$.
\end{remark}

For a general compact set $\Omega$, we can choose $c>0$ and $\mathbf{b}\in\bR^d$ so that $\Omega'=\frac{1}{c}\Omega+\mathbf{b}\subset\left[0,1\right]^d$. If $f\in \sB^{\log}$, we define $g\left(x\right)=f\left(c\left(x-\mathbf{b}\right)\right)$ and consider a network approximating $g$ on $\Omega'$. While this change of variables rescales the input of $f$, the Fourier transform of $g$ satisfies $\norm{g\left(\xi\right)} = c^{-d}\norm{\hat{f}\left(c^{-1}\xi\right)}$, and consequently $\|g\|_{\mathcal{B}^{0}} = \|f\|_{\mathcal{B}^{0}}$. On the other hand, we need the following computation to address the logarithmic factor in $\Norm{g}_{\sB^{\log}}$.
\begin{lemma}\label{lemma:submultiplicative_log}
    If $c\ge 0$ and $t\ge0$, then we have
    \begin{equation*}
        \log_2\left(2+ct\right) \le \log_2\left(2+c\right)\log_2\left(2+t\right).
    \end{equation*}
\end{lemma}
\begin{proof}
    Note that the inequality is equivalent to
    \begin{equation*}
        2+ct \le \left(2+c\right)^{\log_2\left(2+t\right)}.
    \end{equation*}
    Let us define $\phi\left(t\right) = \left(2+c\right)^{\log_2\left(2+t\right)}-ct-2$ and consider $\phi'\left(t\right)$. Then we see that
    \begin{align*}
        \phi'\left(t\right) &= \frac{\ln\left(2+c\right)}{\left(2+t\right)\ln2}\left(2+c\right)^{\log_2\left(2+t\right)} -c\\
        &= \frac{\ln\left(2+c\right)}{\ln 2}\frac{1}{2+t}\left(2+t\right)^{\log_2\left(2+c\right)}-c\\
        &= \frac{\ln\left(2+c\right)}{\ln 2}\left(2+t\right)^{\log_2\left(2+c\right)-1}-c\\
        &\ge \frac{\ln\left(2+c\right)}{\ln 2}\cdot \frac{2+c}{2}-c.
    \end{align*}
    To show $\phi'\left(t\right)\ge0$ for any $c\ge0$, define $\psi\left(c\right)=\frac{\ln\left(2+c\right)}{\ln 2}\cdot \frac{2+c}{2}-c$ and consider $\psi'\left(c\right)$:
    \begin{equation*}
        \psi'\left(c\right) = \frac{1}{2\ln 2}+\frac{\ln\left(2+c\right)}{2\ln 2}-1\ge \frac{1}{2\ln 2}\cdot \ln\left(\frac{e\left(2+c\right)}{4}\right)> 0.
    \end{equation*}
    Therefore, $\phi'\left(t\right)\ge\psi\left(c\right)\ge\psi\left(0\right)=1$ and $\phi$ increases on $\left[0,\infty\right)$. This implies $\phi\left(t\right)\ge0$ and concludes the proof.
\end{proof}

\begin{theorem}\label{thm:L2_convergence_general}
    Suppose that $f\in \sB^{\log}$. For any $m\in\bN$ and a compact set $\Omega\subset\bR^d$, there exists a $\relu$ network $F$ of width $d+4$ and depth $6C_1\left(\Omega\right)m\frac{\Norm{f}_{\sB^{\log}}}{\Norm{f}_{\sB^0}}$ such that
    \begin{equation*}
        \Norm{f-F}_{L^2\left(\Omega\right)} \le \frac{2\pi^2}{\sqrt{m}}\norm{\Omega}^{\frac{1}{2}}\Norm{f}_{\sB^0},
    \end{equation*}
    where $C_1\left(\Omega\right) = \log_2\left(2+\diam\left(\Omega\right)\right)$.
\end{theorem}
\begin{proof}
    Let $c=\diam\left(\Omega\right)$ and choose $\mathbf{b}\in\bR^d$ so that $\Omega'=\frac{1}{c}\Omega+\mathbf{b}\subset\left[0,1\right]^d$. By defining $g\left(x\right)=f\left(c\left(x-\mathbf{b}\right)\right)$, we have
    \begin{align*}
        \Norm{g}_{\sB^{\log}} &= \int_{\bR^d}\log_2\left(2+\norm{\xi}_1\right)\norm{\hat{g}\left(\xi\right)}\diff\xi \\
        &= \int_{\bR^d}\log_2\left(2+\norm{\xi}_1\right)\norm{c^{-d}}\norm{\hat{f}\left(c^{-1}\xi\right)}\diff\xi\\
        &= \int_{\bR^d}\log_2\left(2+ c\norm{\eta}_1\right)\norm{\hat{f}\left(\eta\right)}\diff\eta \\
        &\le \log_2\left(2+c\right)\int_{\bR^d}\log_2\left(2+\norm{\eta}_1\right)\norm{\hat{f}\left(\eta\right)}\diff\eta\\
        &= C_1\left(\Omega\right) \Norm{f}_{\sB^{\log}}.
    \end{align*}
    Since $\Omega'\subset\left[0,1\right]^d$ and $g\in \sB^{\log}$, from Corollary \ref{cor:L2_convergence_unit}, there exists a $\relu$ network $G$ of width $d+4$ and depth $6m\frac{\Norm{g}_{\sB^{\log}}}{\Norm{g}_{\sB^0}}$ such that
    \begin{equation*}
        \left\Vert g-G\right\Vert_{L^2\left(\Omega'\right)} \le \frac{2\pi^2}{\sqrt{m}}\norm{\Omega'}\Norm{g}_{\sB^0}.
    \end{equation*}
    As a composition of affine transforms is again an affine transform, adjusting the weight and bias in the first layer of $G$ yields a network $F\left(x\right)=G\left(c^{-1}x+\mathbf{b}\right)$, and we see that
    \begin{align*}
        \Norm{f - F}_{L^2\left(\Omega\right)}^2 &= \int_{\Omega} \left\vert f\left(x\right)-F\left(x\right)\right\vert^2\diff x\\
        &=\int_{\frac{1}{c}\Omega+\mathbf{b}} \left\vert f\left(c\left(y-\mathbf{b}\right)\right) - F\left(c\left(y-\mathbf{b}\right)\right)\right\vert^2 c^d\diff y\\
        &= c^d \int_{\Omega'}\left\vert g\left(y\right)-G\left(y\right)\right\vert^2 \diff y\\
        &= c^d \Norm{g-G}_{L^2\left(\Omega'\right)}^2\\
        &\le \frac{4\pi^4}{\sqrt{m}}c^d\norm{\Omega'} \Norm{g}_{\sB^0}^2\\
        &= \frac{4\pi^4}{\sqrt{m}}\norm{\Omega}\Norm{f}_{\sB^0}^2.
    \end{align*}
    Note here that the depth of $F$ is $6m\frac{\Norm{g}_{\sB^{\log}}}{\Norm{g}_{\sB^0}}\le 6cm\frac{\Norm{f}_{\sB^{\log}}}{\Norm{f}_{\sB^0}}$.
\end{proof}

\begin{remark}
    Given an error bound for a two-layer neural network, application of Lemma~\ref{lemma:wide_to_deep} yields the same order of error for a deep network. For example, if $f\in \sB^1$, the results of \cite{barron1994approximation} indicate that a network with width $d+2$ and depth $N$ achieves an $L^2$ error of order $N^{-\frac{1}{2}}$. However, the direct transformation from a wide to a deep network does not clarify the mechanisms underlying the improved effectiveness of deep networks. The present result demonstrates that a deep network mitigates the regularity condition on the target function, with the constant term in the error bound determined by $\Norm{f}_{\sB^0}$. Thus, the error depends solely on $\Norm{f}_{\sB^0}$, and the bound holds independently of the decay rate of the Fourier amplitude of $f$. The decay rate instead influences the required network depth.
\end{remark}

Assuming higher regularity for the target function $f$, the universal approximation theorem is known to hold in Sobolev spaces. In the next section, we introduce the high-order log-Barron space $\sB^{s,\log}$ and address the error bound with respect to the Sobolev norm $\Norm{\cdot}_{H^s}$.

\section{$H^1$ convergence}\label{sec:H1_convergence}
Extending the previous results to Sobolev approximation requires higher regularity of the target function $f$. In accordance with the classical Barron framework, this regularity is imposed by a polynomial decay rate on the Fourier amplitude. The \textit{log-Barron space} $\sB^{s,\log}$ for $s>0$ is defined as follows:
\begin{align*}
    \left\Vert f\right\Vert_{\sB^{s,\log}} &=  \int_{\bR^d}\left(1+\norm{\xi}_1^s\right)\log_2\left(2+\norm{\xi}_1\right)\left\vert\hat{f}\left(\xi\right)\right\vert\diff\xi,\\
    \sB^{s,\log} &= \left\{ f\in \sS(\bR^d): \Norm{f}_{\sB^{s,\log}}<\infty\right\}.
\end{align*}
An argument analogous to that for $\sB^{\log}$ establishes the completeness and embedding relations of $\sB^{s,\log}$ with the Sobolev spaces. The details are omitted as the proof is redundant. 

Since $\sB^{1,\log}$ requires a slightly faster amplitude decay than $\sB^1$, any $f\in \sB^{1,\log}$ can be approximated by a deep $\relu$ network with respect to the Sobolev norm $\Norm{\cdot}_{H^1\left(\Omega\right)}$. For $f\in \sB^{1,\log}$ and compact set $\Omega\subset\left[0,1\right]^d$, we deduce a Sobolev approximation using similar steps. Note that, as we can see from Table \ref{tab:summary}, this is a weaker regularity assumption than the one used in \cite{siegel2020approximation}.

\begin{theorem}\label{thm:H1_convergence_unit}
    Suppose $f\in \sB^{1,\log}$. For any $m\in\bN$ and a compact set $\Omega\subset\left[0,1\right]^d$, there exist ReLU networks $F_1,\ldots,F_m$ of width 3 and depth $L_i$ such that
    \begin{equation*}
        \left\Vert f-\frac{1}{m}\sum_{i=1}^m F_i\right\Vert_{H^1\left(\Omega\right)}^2 \le \frac{11\pi^4}{m}\norm{\Omega}\Norm{f}_{\sB^1}^2
    \end{equation*}
    and
    \begin{equation*}
        \sum_{i=1}^mL_i \le 5m\frac{\Norm{f}_{\sB^{1,\log}}}{\Norm{f}_{\sB^1}}.
    \end{equation*}
\end{theorem}
\begin{proof}
    Let $\mu$ and $F$ be defined similarly to the proof of Theorem~\ref{thm:L2_convergence_unit}:
    \begin{align*}
        \diff\mu\left(\xi,r\right) &= \frac{1}{\Norm{f}_{\sB^1}}\mathbf{1}_{\left[-\frac{1}{2},\frac{1}{2}\right]}\left(r\right)\left(1+\norm{\xi}_1\right) \norm{\hat{f}\left(\xi\right)} \diff\xi \diff r,\\
        F\left(x;\xi,r\right) &= -2\pi^2\Norm{f}_{\sB^1}\left(1+\norm{\xi}_1\right)^{-1}\cos\left(2\pi r\right)\gamma_{,n_{\xi}}\left(t_{\xi}\left(x\right),r\right).
    \end{align*}
    For i.i.d. samples $\left\{\left(\xi_i,r_i\right)\right\}_{i=1}^m$, $\mu^m$ refers the distribution and $\bar{F}$ is the sample mean of $F\left(x;\xi_i,r_i\right)$. We use the Sobolev norm to define a random variable X,
    \begin{equation*}
        X = \left\Vert f-\bar{F}\right\Vert_{H^1\left(\Omega\right)}^2 = \left\Vert f-\bar{F}\right\Vert_{L^2\left(\Omega\right)}^2 + \sum_{j=1}^d\left\Vert D_jf - D_j\bar{F}\right\Vert_{L^2\left(\Omega\right)}^2.
    \end{equation*}
    For each $\xi=\left(\xi^{(1)},\ldots,\xi^{(d)}\right)\in\bR^d$, $r\in\bR$, and $j=1,\ldots,d$, we have
    \begin{align}
        D_j F\left(x;\xi,r\right) &= -2\pi^2\Norm{f}_{\sB^1}\left(1+\norm{\xi}_1\right)^{-1}\cos\left(2\pi r\right)D_j\left(\gamma_{,n_{\xi}}\left(t_{\xi},r\right)\right) \nonumber\\
        &= -2\pi^2\Norm{f}_{\sB^1}\left(1+\norm{\xi}_1\right)^{-1}\cos\left(2\pi r\right) D_j\left(\gamma\left(\xi\cdot x+\theta\left(\xi\right) \mod 1, r\right) \right) \nonumber \\
        &= -2\pi^2\Norm{f}_{\sB^1}\left(1+\norm{\xi}_1\right)^{-1}\cos\left(2\pi r\right) \xi^{(j)} D\gamma \left(\xi\cdot x+\theta\left(\xi\right) \mod 1, r\right). \label{eqn:DjF} 
    \end{align}
    Hence, we can bound $D_jF\left(x;\xi,r\right)$ for $x\in\Omega$ by
    \begin{equation*}
        \left\vert D_jF\left(x;\xi,r\right)\right\vert \le 2\pi^2\Norm{f}_{\sB^1}\left(1+\norm{\xi}_1\right)^{-1}\norm{\xi^{(j)}}.
    \end{equation*}
    Next, let us consider the i.i.d. samples $\left\{\left(\xi_i,r_i\right)\right\}_{i=1}^m$ from $\mu^m$. Then we may write
    \begin{equation*}
        \bE_{\mu^m}\left[X\right] = \bE_{\mu^m}\left[\left\Vert f-\bar{F}\right\Vert_{L^2\left(\Omega\right)}^2\right] + \bE_{\mu^m}\bigg[\sum_{j=1}^d \left\Vert D_jf-D_j\bar{F}\right\Vert_{L^2\left(\Omega\right)}^2\bigg].
    \end{equation*}
    We can bound the first term following similar steps to the previous theorem:
    \begin{align*}
        \bE_{\mu^m}\left[\left\Vert f-\bar{F}\right\Vert_{L^2\left(\Omega\right)}^2\right] &\le \frac{1}{m}\int_{\Omega}\bE_{\mu}\left[ F\left(x;\xi,r\right)^2\right]  \diff x\\
        &\le \frac{1}{m}\int_{\Omega} \sup_{\xi,r}\left\vert F\left(x;\xi,r\right)\right\vert^2 \diff x\\
        &\le \frac{\pi^4}{m}\norm{\Omega}\Norm{f}_{\sB^1}^2.
    \end{align*}
    For the other term, since $\bE_{\mu^m}\left[D_j\bar{F}\right] = D_j\bE_{\mu^m}\left[\bar{F}\right]$, we see that
    \begin{align*}
        \bE_{\mu^m}\bigg[\sum_{j=1}^m\left\Vert D_jf-D_j\bar{F}\right\Vert_{L^2\left(\Omega\right)}^2\bigg] &= \int_{\Omega} \sum_{j=1}^m \int \left\vert D_jf\left(x\right)-D_j\bar{F}\left(x\right)\right\vert^2 \diff \mu^m\left(\left\{\xi_i,r_i\right\}\right) \diff x \\
        &= \int_{\Omega} \sum_{j=1}^m \int \left\vert D_j\bE_{\mu^m}\left[\bar{F}\left(x\right)\right] - D_j\bar{F}\left(x\right)\right\vert^2 \diff \mu\left(\left\{\xi_i,r_i\right\}\right) \diff x\\
        &= \int_{\Omega}\sum_{j=1}^m \Var_{\mu^m}\left[ D_j\bar{F}\left(x\right)\right] \diff x\\
        &\le \int_{\Omega}\sum_{j=1}^m \frac{1}{m}\Var_{\mu}\left[D_j F\left(x;\xi,r\right)\right] \diff x \\
        &\le \frac{1}{m} \int_{\Omega}\sum_{j=1}^m \bE_{\mu}\left[\left\vert D_jF\left(x;\xi,r\right) \right\vert^2\right\vert \diff x\\
        &\le \frac{1}{m} \int_{\Omega}\sum_{j=1}^m 4\pi^4\Norm{f}_{\sB^1}^2 \bE_{\mu}\left[\left(1+\norm{\xi}_1\right)^{-2}\norm{\xi^{(j)}}^2\right] \diff x\\
        &= \frac{4\pi^4}{m}\Norm{f}_{\sB^1}\int_{\Omega}\sum_{j=1}^m \int_{\bR^d} \left(1+\norm{\xi}_1\right)^{-1}\norm{\xi^{(j)}}^2\norm{\hat{f}\left(\xi\right)} \diff \xi\\
        &\le \frac{4\pi^4}{m}\Norm{f}_{\sB^1}\int_{\Omega} \int_{\bR^d}\left(1+\norm{\xi}_1\right)^{-1}\norm{\xi}_1^2\norm{\hat{f}\left(\xi\right)} \diff \xi\\
        &\le \frac{4\pi^4}{m}\Norm{f}_{\sB^1}\int_{\Omega} \int_{\bR^d}\left(1+\norm{\xi}_1\right) \norm{\hat{f}\left(\xi\right)} \diff\xi\\
        &= \frac{4\pi^4}{m}\norm{\Omega}\Norm{f}_{\sB^1}^2.
    \end{align*}
    Therefore, we attain
    \begin{equation*}
        \bE_{\mu^m}\left[X\right] \le \frac{5\pi^2}{m}\norm{\Omega}\Norm{f}_{\sB^1}^2,
    \end{equation*}
    and Markov's inequality induces
    \begin{equation}\label{eqn:Markov_H1error}
        \bP_{\mu^m}\left[ X < \frac{5\left(2+\eps_1\right)\pi^4}{m}\norm{\Omega} \Norm{f}_{\sB^1}^2\right] \ge \bP_{\mu^m}\left[ X < \left(2+\eps_1\right)\bE_{\mu^m}\left[X\right]\right] \ge \frac{1+\eps_1}{2+\eps_1}.
    \end{equation}

    Note that $F\left(x;\xi_i,r_i\right)$ is implemented by the same network as we did in the previous theorem.
    Thus, the expected sum of depths $L=\sum_{i=1}^m L_{\xi_i}$ is estimated as
    \begin{align*}
        \bE_{\mu^m}\left[L\right] & \le 2m\bE_{\mu}\left[\log_2\left(2+\norm{\xi}_1\right)\right] \\
        &= \frac{2m}{\Norm{f}_{\sB^1}}\int_{\bR^d}\left(1+\norm{\xi}_1\right)\log_2\left(2+\norm{\xi}_1\right)\norm{\hat{f}\left(\xi\right)} \diff \xi\\
        &= \frac{2m}{\Norm{f}_{\sB^1}}\Norm{f}_{\sB^{1,\log}}.
    \end{align*}
    Markov's inequality implies that
    \begin{equation}\label{eqn:Markov_H1depth}
        \bP_{\mu^m}\left[ L < \frac{2\left(2+\eps_2\right)m}{\Norm{f}_{\sB^1}}\norm{\Omega}\Norm{f}_{\sB^{1,\log}}\right] \ge \bP_{\mu^m}\left[L < \left(2+\eps_2\right)\bE_{\mu^m}\left[L\right]\right] = \frac{1+\eps_2}{2+\eps_2}.
    \end{equation}
    Both \eqref{eqn:Markov_H1error} and \eqref{eqn:Markov_H1depth} hold with a probability of at least 
    \begin{align*}
        & \bP_{\mu^m}\left[\left(X < \frac{5\left(2+\eps_1\right)\pi^4}{m}\norm{\Omega}\Norm{f}_{\sB^1}^2\right) \text{ and } \left(L < \frac{2\left(2+\eps_2\right)m}{\Norm{f}_{\sB^1}}\Norm{f}_{\sB^{1,\log}}\right)\right]\\
        &= \bP_{\mu^m}\left[X < \frac{5\left(2+\eps_1\right)\pi^4}{m}\norm{\Omega}\Norm{f}_{\sB^1}^2\right] + \bP_{\mu^m}\left[L < \frac{2\left(2+\eps_2\right)m}{\Norm{f}_{\sB^1}}\Norm{f}_{\sB^{1,\log}}\right]\\
        &\hspace{4mm} - \bP_{\mu^m}\left[\left(X < \frac{5\left(2+\eps_1\right)\pi^4}{m}\norm{\Omega}\Norm{f}_{\sB^1}^2\right) \text{ or } \left(L < \frac{2\left(2+\eps_2\right)m}{\Norm{f}_{\sB^1}}\Norm{f}_{\sB^{1,\log}}\right)\right]\\
        &\ge \frac{1+\eps_1}{2+\eps_1}+\frac{1+\eps_2}{2+\eps_2}-1\\
        &> 0.
    \end{align*}
    Since $\eps_1,\eps_2$ can be arbitrary, the choice $\eps_1=0.2$ and $\eps_2=0.5$ concludes the proof.
\end{proof}

\begin{remark}
    Note that $\sum_{j=1}^d \left\vert D_jF\right\vert$ is bounded by $\Norm{f}_{\sB^1}$ and $\norm{D\gamma}$ up to a multiplicative constant. Indeed, our proof strategy can be extended to more general functions $\gamma$ under two necessary conditions: The derivatives of $\gamma\left(\cdot,r\right)$ are uniformly bounded with respect to $r$, and $\gamma\left(\cdot,r\right)$ is exactly represented by a neural network. Under these assumptions, it is plausible that the present analysis can be generalized to higher-order Sobolev approximation for deep neural networks, provided that $\gamma\left(\cdot,r\right)\in H^n$ and a decomposition lemma analogous to Lemma~\ref{lemma:cos_2pint} is available. Moreover, exact network representations of $\gamma$ suggest the possibility of deriving approximation error bounds in stronger norms, such as $L^\infty(\Omega)$ or $W^{n,\infty}(\Omega)$, beyond the $L^2$ and $H^1$ considered in this work. In this paper, we focus on $\relu$ networks and leave these extensions for future research.
\end{remark}

Combining the networks $F_i$ in Theorem \ref{thm:H1_convergence_unit} into a single deep network produces a Sobolev approximation corresponding to Corollary~\ref{cor:L2_convergence_unit}.
\begin{corollary}\label{cor:H1_convergence_unit}
    Suppose $d\in\bN$ and $f\in \sB^{1,\log}$. For a compact set $\Omega\subset\left[0,1\right]^d$ and any $m\in\bN$, there exists a ReLU network $F$ of width $d+4$ and depth $6m \frac{\Norm{f}_{\sB^{1,\log}}}{\Norm{f}_{\sB^1}}$ such that
    \begin{equation*}
        \left\Vert f - F\right\Vert_{H^1\left(\Omega\right)} \le \frac{4\pi^2}{\sqrt{m}}\norm{\Omega}^{\frac{1}{2}}\Norm{f}_{\sB^1}.
    \end{equation*}
\end{corollary}

In the general case, we can use scaling to define $g\left(x\right)=f\left(c\left(x-\mathbf{b}\right)\right)$ for $c\ge1$ and $\mathbf{b}\in\bR^d$ such that $\Omega'=\frac{1}{c}\Omega+\mathbf{b}\subset\left[0,1\right]^d$. As in the previous section, it is enough to approximate $g$ on $\Omega'$ using $\Norm{g}_{\sB^{1,\log}}$. In this case, we attain inequality $\Norm{f}_{\sB^1}\le \Norm{g}_{\sB^1} \le c\Norm{f}_{\sB^1}$ as follows:
\begin{align*}
    \Norm{g}_{\sB^1} &= \int_{\bR^d}\left(1+\norm{\xi}_1\right)\norm{\hat{g}\left(\xi\right)} \diff\xi \\
    &= \int_{\bR^d}\left(1+\norm{\xi}_1\right)c^{-d}\left\vert\hat{f}\left(c^{-1}\xi\right)\right\vert \diff\xi\\
    &= \int_{\bR^d}\left(1+c\norm{\eta}_1\right)\norm{\hat{f}\left(\eta\right)}\diff\eta. \quad \left(\eta=c^{-1}\xi\right)
\end{align*}
Using the Lemma~\ref{lemma:submultiplicative_log} to bound $\Norm{g}_{\sB^{1,\log}}$, we can approximate $g$ on $\Omega'$ and thereby $f$ on $\Omega$.

\begin{theorem}\label{thm:H1_convergence_general}
    Suppose $f\in \sB^{1,\log}$. For any $m\in\bN$ and a compact set $\Omega\subset\bR^d$, there exists a $\relu$ network $F$ of width $d+4$ and depth $6mC_2\log_2\left(2+C_2\right)\frac{\Norm{f}_{\sB^{1,\log}}}{\Norm{f}_{\sB^1}}$ such that
    \begin{equation*}
        \Norm{f-F}_{H^1\left(\Omega\right)}\le\frac{4\pi^2}{\sqrt{m}}C_2\norm{\Omega}^{\frac{1}{2}}\Norm{f}_{\sB^1},
    \end{equation*}
    where $C_2=C_2\left(\Omega\right)=\max\left\{1,\diam\left(\Omega\right)\right\}$.
\end{theorem}
\begin{proof}
    Let $c=\max\left\{1,\diam\left(\Omega\right)\right\}$ and choose $\mathbf{b}\in\bR^d$ so that $\Omega'=\frac{1}{c}\Omega+\mathbf{b}\subset\left[0,1\right]^d$. Define $g\left(x\right) = f\left(c\left(x-\mathbf{b}\right)\right)$ and apply Corollary \ref{cor:H1_convergence_unit} to find a network $G$ of width $d+4$ and depth $6\frac{\Norm{g}_{\sB^{1,\log}}}{\Norm{g}_{\sB^1}}$ such that
    \begin{equation*}
        \Norm{g-G}_{H^1\left(\Omega'\right)}\le \frac{4\pi^2}{\sqrt{m}}\norm{\Omega'}^{\frac{1}{2}}\Norm{g}_{\sB^1}.
    \end{equation*}
    Note that a function $F\left(x\right)\coloneqq G\left(c^{-1}x+\mathbf{b}\right)$ defined on $\Omega$ is a network of the same architecture of $g$. Therefore, we have
    \begin{align*}
        \Norm{f-F}_{L^2\left(\Omega\right)}^2 &= \int_{\Omega}\left\vert f\left(x\right)-F\left(x\right)\right\vert^2 \diff x\\
        &= \int_{\Omega}\left\vert g\left(c^{-1}x+\mathbf{b}\right)-G\left(c^{-1}x+\mathbf{b}\right)\right\vert^2 \diff x\\
        &= \int_{\Omega'} \left\vert g\left(y\right)-G\left(y\right)\right\vert^2 c^d\diff y \\
        &= c^d \Norm{g-G}_{L^2\left(\Omega'\right)}^2,\\
    \end{align*}
    and
    \begin{align*}
        \Norm{D_jf-D_jF}_{L^2\left(\Omega\right)}^2 &= \int_{\Omega}\left\vert D_jf\left(x\right)-D_jF\left(x\right)\right\vert^2\diff x\\
        &= \int_{\Omega}c^{-2}\left\vert D_jg\left(c^{-1}x+\mathbf{b}\right)-D_jG\left(c^{-1}x+\mathbf{b}\right)\right\vert^2 \diff x\\
        &= c^{-2}\int_{\Omega'}\left\vert D_jg\left(y\right) - D_jG\left(y\right)\right\vert^2 c^d\diff y\\
        &= c^{d-2} \Norm{D_jg-D_jG}_{L^2\left(\Omega'\right)}^2.
    \end{align*}
    Then we have
    \begin{align*}
        \Norm{f-F}_{H^1\left(\Omega\right)}^2 &= \Norm{f-F}_{L^2\left(\Omega\right)}^2 + \sum_{j=1}^d \Norm{D_jf-D_jF}_{L^2\left(\Omega\right)}^2\\
        &= c^d \Norm{g-G}_{L^2\left(\Omega'\right)}^2 + c^{d-2}\sum_{j=1}^d\Norm{D_jg-D_jG}_{L^2\left(\Omega'\right)}^2\\
        &\le c^d\Norm{g-G}_{H^1\left(\Omega'\right)}^2 \\
        &\le c^d \frac{16\pi^4}{m}\norm{\Omega'}\Norm{g}_{\sB^1}\\
        &\le c\frac{16\pi^4}{m}\norm{\Omega}\Norm{f}_{\sB^1}.
    \end{align*}
    Finally, the depth of $F$ is bounded by
    \begin{align*}
        6m\frac{\Norm{g}_{\sB^{1,\log}}}{\Norm{g}_{\sB^1}} &= 6m\frac{\int_{\bR^d}\left(1+c\norm{\eta}_1\right)\log_2\left(2+c\norm{\eta}_1\right)\norm{\hat{f}\left(\eta\right)}\diff \eta}{\int_{\bR^d}\left(1+c\norm{\eta}_1\right)\norm{\hat{f}\left(\eta\right)} \diff\eta}\\
        &\le 6m c\log_2\left(2+c\right)\frac{\int_{\bR^d}\left(1+\norm{\eta}_1\right)\log_2\left(2+\norm{\eta}_1\right)\norm{\hat{f}\left(\eta\right)}\diff \eta}{\int_{\bR^d}\left(1+\norm{\eta}_1\right)\norm{\hat{f}\left(\eta\right)} \diff\eta}\\
        &= 6mc\log_2\left(2+c\right) \frac{\Norm{f}_{\sB^{1,\log}}}{\Norm{f}_{\sB^1}}.
    \end{align*}
\end{proof}

\section{Conclusion}\label{sec:conclusion}
This study introduces the log-Barron space and establishes dimension-independent approximation error bounds for target functions in this space using deep $\relu$ networks with bounded width. We analyze the functional-analytic properties of the proposed space, demonstrating that it forms a Banach space and clarifying its embedding relations with classical Sobolev spaces. We also provide a generalization bound by obtaining an upper bound on the corresponding Rademacher complexity. Our results extend the classical Barron framework from $\sB^s$ to $\sB^{\log}$, resulting in improved upper bounds with respect to both regularity and the required number of parameters. 

A notable contribution of this study is that dimension-independent approximation error is analyzed explicitly as a function of network depth. While existing Barron-type results have achieved dimension-independent rates by increasing width, our results provide the quantitative approximation error bound under a fixed-width architecture, where accuracy improves as depth increases. The proposed construction shows that the approximation error depends only on the zero-order Barron norm $\Norm{f}_{\sB^0}$, whereas the required network depth is controlled by the log-Barron norm $\Norm{f}_{\sB^{\log}}$. As a result, the decay rate of the Fourier spectrum influences only the required depth, while the error bound itself depends on the total spectral amplitude. This theoretical result provides justification for employing deeper architectures when approximating functions with high-frequency features.

Beyond the error bound, the embedding analysis of the log-Barron space also enables us to quantify the complexity of the associated function class. In particular, we compute the Rademacher complexity, which is commonly used to control generalization error. This result complements the approximation analysis by characterizing the size of the hypothesis class associated with the proposed function space.

From a structural view, the ensemble-based construction allows multiple sub-networks to be combined into a single deep narrow network, thereby establishing a precise connection between width-based and depth-based approximation strategies. In contrast to classical Barron-type theories, where efficiency is obtained by widening the network, our results demonstrate that increasing depth alone can serve as an effective mechanism for achieving dimension-independent approximation.

We conclude by discussing several limitations and directions for future research. In the present construction, the first $d$ neurons in each hidden layer are used to preserve input information. Since the approximation error is measured in $L^2$ or $H^1$ norms, it may be possible to further reduce the required width by incorporating encoding schemes proposed in \cite{Park_2021}. A systematic study of minimal width conditions in Barron-type spaces would enhance understanding of the interplay between network architecture and function regularity. Moreover, following \cite{liao2025spectral}, the composition of $\beta$ functions is exploited to represent high-frequency functions. Although each layer that computes $\beta$ is not differentiable, the overall composition $F$ is smooth. Consequently, an ensemble of $F$ can enable high-order approximation using smooth activation functions such as sigmoid or $\relu^k$. Finally, from the perspective of numerical analysis, a faster convergence rate in the $L^2$ sense is expected when the target function exhibits higher regularity. Investigating improved error bounds for functions in $\sB^s$ or $\sB^{s,\log}$ represents a promising direction for future research.

\bibliography{references}
\bibliographystyle{abbrv}

\end{document}